\documentclass[11pt]{article}

\usepackage[letterpaper]{geometry}
\usepackage[parfill]{parskip}

\usepackage{times}
\usepackage{graphicx} 

\usepackage[numbers,sort&compress]{natbib}

\usepackage{algorithm}
\usepackage{algorithmic}

\usepackage{tablefootnote}

\usepackage{hyperref}

\usepackage{authblk}
\usepackage{comment}

\usepackage{amsfonts,amsmath,amssymb,amsthm,bm}
\usepackage{xspace}
\usepackage{caption}
\usepackage{subcaption}
\usepackage{enumitem}
\usepackage{appendix}

\newcommand{\R}{\mathbb{R}}
\newcommand{\nlsum}{\sum\nolimits}
\newcommand{\nlprod}{\prod\nolimits}
\newcommand{\norm}[1]{\|#1\|}

\newcommand{\ip}[2]{\left\langle #1, #2\right\rangle}
\newtheorem{theorem}{Theorem}
\newtheorem{proposition}{Proposition}
\newtheorem{lemma}{Lemma}

\theoremstyle{definition}

\newtheorem{example}{Example}
\newtheorem{remark}{Remark}

\newcommand{\argmax}{\mathop{\mathrm{argmax}}}

\def\E{\mathbb{E}}
\def\P{\mathbb{P}}

\def\Var{\mathrm{Var}}

\def\tr{\mathrm{tr}}

\def\R{\mathbb{R}}
\def\cA{\mathcal{A}}

\def\cH{\mathcal{H}}
\def\cM{\mathcal{M}}

\def\cX{\mathcal{X}}
\def\cY{\mathcal{Y}}

\def\x{\mathbf{x}}
\def\y{\mathbf{y}}
\def\w{\mathbf{w}}

\newcommand{\bcfw}{\textsc{Bcfw}\xspace}
\newcommand{\algo}{\textsc{Ap-Bcfw}\xspace}
\newcommand{\spalgo}{\textsc{Sp-Bcfw}\xspace}

\begin{document}

\title{Parallel and Distributed Block-Coordinate Frank-Wolfe Algorithms}

\author[1,2]{Yu-Xiang Wang}
\author[1]{Veeranjaneyulu Sadhanala}
\author[1]{Wei Dai}
\author[1]{Willie Neiswanger}
\author[3]{Suvrit Sra}
\author[1]{Eric P. Xing}
\affil[1]{Machine Learning Department, Carnegie Mellon University}
\affil[2]{Department of Statistics, Carnegie Mellon University}
\affil[3]{LIDS, Massachusetts Institute of Technology}

\maketitle
\begin{abstract} 
  We develop parallel and distributed Frank-Wolfe algorithms; the former on shared memory machines with mini-batching, and the latter in a delayed update framework. Whenever possible, we perform computations asynchronously, which helps attain speedups on multicore machines as well as in distributed environments. Moreover, instead of worst-case bounded delays, our methods only depend (mildly) on \emph{expected} delays, allowing them to be robust to stragglers and faulty worker threads. Our algorithms assume block-separable constraints, and subsume the recent Block-Coordinate Frank-Wolfe (BCFW) method~\citep{lacoste2013block}.  Our analysis reveals problem-dependent quantities that govern the speedups of our methods over BCFW. We present experiments on structural SVM and Group Fused Lasso, obtaining significant speedups over competing state-of-the-art (and synchronous) methods.
\end{abstract}

\section{Introduction}
The classical Frank-Wolfe (FW) algorithm~\citep{frank1956algorithm} has witnessed a huge surge of interest recently~\citep{clarkson2010coresets,jaggi2011sparse,jaggi2013revisiting,bach2013everywhere}. The FW algorithm iteratively solves the problem
\begin{equation}
  \label{eq:5}
  \min_{x \in \cM}\quad f(x),
\end{equation}
where $f$ is a smooth function (typically convex) and $\cM$ is a closed convex set. The key factor that makes FW appealing is its use of a \emph{linear oracle} that solves $\min_{x\in \cM}\ip{x}{g}$, instead of a projection (quadratic) oracle that solves $\min_{x \in \cM}\norm{x-c}$, especially because the linear oracle can be much simpler and faster. 

This appeal has motivated several new variants of basic FW, e.g., regularized FW \citep{zhang2013polar,bredies2009generalized,harchaoui2013conditional},  linearly convergent special cases~\citep{lacoste2015global,garber2013linearly},  stochastic/online versions~\citep{ouyang2010fast,hazanKale12,lafond2015convergence}, and a randomized block-coordinate FW~\citep{lacoste2013block}.

But despite this progress, parallel and distributed FW variants are barely studied. In this work, we develop new parallel and distributed FW algorithms, in particular for \emph{block-separable} instances of~\eqref{eq:5} that assume the form
\begin{equation}
  \begin{split}
    \label{eq:1}
    \min_{x}\ f(x)\ \text{s.t.}\ x=[x_{(1)},...,x_{(n)}]\in \nlprod_{i=1}^n \cM_i,
  \end{split}
\end{equation}
where $\cM_i \subset \R^{m_i}$ ($1\le i\le n$) is a compact convex set and $x_{(i)}$ are coordinate blocks of $x$. This setting for FW was considered in~\citep{lacoste2013block}, who introduced the Block-Coordinate Frank-Wolfe (\bcfw) method.

Such problems arise in many applications, notably, structural SVMs~\citep{lacoste2013block}, routing~\citep{leblanc1975efficient}, group fused lasso~\citep{alaiz2013group,bleakley2011group}, trace-norm based tensor completion~\citep{liu2013tensor}, reduced rank nonparametric regression~\citep{foygel2012nonparametric}, and structured submodular minimization~\citep{jegelka2013reflection}, among others.

One approach to solve~\eqref{eq:1} is via block-coordinate (gradient) descent (BCD), which forms a local quadratic model for a block of variables, and then solves a \emph{projection} subproblem~\citep{nesterov2012efficiency,richtarik2012parallel,beck2013convergence}. However, for many problems, including the ones noted above, projection can be expensive (e.g., projecting onto the trace norm ball, onto base polytopes~\citep{fujishige2011submodular}), or even computationally intractable~\citep{collins2008exponentiated}.  

Frank-Wolfe (FW) methods excel in such scenarios as they rely only on linear oracles solving $\min_{s\in \cM}\langle s,\nabla f(\cdot)\rangle$. 
For $\cM = \prod_i \cM_i$, this breaks into the $n$ independent problems
\begin{equation}\label{eq:exact_oracle}
\min_{s_{(i)} \in \cM_i}\quad\langle s_{(i)}, \nabla_{(i)} f(x)\rangle,\quad 1 \le i \le n,
\end{equation}
where $\nabla_{(i)}$ denotes the gradient w.r.t.~the coordinates $x_{(i)}$. 
It is immediate that these $n$ subproblems can be solved in parallel (an idea dating back to at least~\citep{leblanc1975efficient}). But there is a practical impediment: updating all the coordinates at each iteration (serially or in parallel) is expensive hampering use of FW on big-data problems. 

This drawback is partially ameliorated by \bcfw~\citep{lacoste2013block}, a method that randomly selects a block $\cM_i$ at each iteration and performs FW updates with it. However, this procedure is \emph{strictly} sequential: it does not take advantage of modern multicore architectures or of high-performance distributed clusters. 

\textbf{Contributions.} In light of the above, we develop scalable FW methods, and make the following main contributions: 
\begin{list}{$\bullet$}{\leftmargin=1em}
  \setlength{\itemsep}{-1pt}
\item Parallel and distributed block-coordinate Frank-Wolfe algorithms, henceforth both referred as \algo, that allow asynchronous computation. \algo depends only (mildly) on the \emph{expected} delay, therefore is robust to stragglers and faulty worker threads. 
\item An analysis of the primal and primal-dual convergence of \algo and its variants for any minibatch size and potentially unbounded maximum delay. 
When the maximum delay is actually bounded, we show stronger results using results from load-balancing on max-load bounds.


%
\item Insightful deterministic conditions under which minibatching \emph{provably} improves the convergence rate for a class of problems (sometimes by orders of magnitude).
\item Experiments that demonstrate on real data how our algorithm solves a structural SVM problem several times faster than the state-of-the-art.
\end{list}
In short, our results contribute to making FW more attractive for big-data applications. To lend further perspective, we  compare our methods to some closely related works below. Space limits our summary; we refer the reader to~\citet{jaggi2013revisiting,zhang2012accelerated,lacoste2013block,freund2014new} for additional notes and references.

%
\textbf{\bcfw and Structural SVM.}
Our algorithm \algo extends and generalizes \bcfw to parallel computation using mini-batches. Our convergence analysis follows the proof structure in \citet{lacoste2013block}, but with different stepsizes that must be carefully chosen. Our results contain \bcfw as a special case. A large portion of \citet{lacoste2013block} focuses on more explicit (and stronger) guarantee for \bcfw on structural SVM. While we mainly focus on a more general class of problems, the particular subroutine needed by structural SVM requires special treatment; we discuss the details in Appendix~\ref{structSVM}. 

\textbf{Parallelization of sequential algorithms.}
The idea of parallelizing sequential optimization algorithms is not new. It dates back to \citep{tsitsiklis1986distributed} for stochastic gradient methods; more recently \citet{richtarik2012parallel,liu2013asynchronousCD,strads} study parallelization of BCD. The conditions under which these parallel BCD methods succeed, e.g., expected separable overapproximation (ESO), and coordinate Lipschitz conditions, bear a close resemblance to our conditions in Section~\ref{sec:speedup}, but are not the same due to differences in how solutions are updated and what subproblems  arise. In particular, our conditions are \emph{affine invariant}.
We provide detailed comparisons to parallel coordinate descents in Appendix~\ref{sec:cmp_richtarik}.

\textbf{Asynchronous algorithms.} Asynchronous algorithms that allow delayed parameter updates have been proposed earlier for stochastic gradient descent~\citep{niu2011hogwild} and parallel BCD~\citep{liu2013asynchronousCD}. We propose the first asynchronous algorithm for Frank-Wolfe. Our asynchronous scheme not only permits delayed minibatch updates, but also allows the updates for coordinate blocks \emph{within each} minibatch to have different delays. Therefore, each update may not be a solution of \eqref{eq:exact_oracle} for any single $x$. In addition, we obtained strictly better dependency on the delay parameter than predecessors (e.g., an exponential improvement over \citet{liu2013asynchronousCD}) possibly due to a sharper analysis. 


\textbf{Other related work.} While preparing our manuscript, we discovered the preprint \citep{bellet2014distributed} which also studies distributed Frank-Wolfe. We note that \citep{bellet2014distributed} focuses on Lasso type problems and communication costs, and hence, is not directly comparable to our results.


\textbf{Notation.} We briefly summarize our notation now. The vector $x\in \R^m$ denotes the parameter vector, possibly split into $n$ coordinate blocks. For block $i=1,...,n$, $E_i \in \R^{m \times m_i}$ is the projection matrix which projects $x \in \R^m$ down to $x_{(i)} \in \R^{m_i}$; thus $x_{(i)} = E_ix$. The adjoint operator $E_i^*$ maps $\R^{m_i} \to \R^m$, thus $x_{[i]} = E_i^*x_{(i)}$ is $x$ with zeros in all dimensions except $x_{(i)}$ (note the subscript $x_{[i]}$).
We denote the size of a minibatch by $\tau$, and the number of parallel workers (threads) by $T$. Unless otherwise stated, $k$ denotes the iteration/epoch counter and $\gamma$ denotes a stepsize. Finally, $C_f^{\tau}$ (and other such constants) denotes some curvature measure associated with function $f$ and minibatch size $\tau$. Such constants are important in our analysis, and will be described in greater detail in the main text.


%
%

\section{Algorithm}
\label{sec:alg}
In this section, we develop and analyze an asynchronous parallel block-coordinate Frank-Wolfe algorithm, hereafter \algo, to solve \eqref{eq:1}.

Our algorithm is designed to run fully asynchronously on either a shared-memory multicore architecture or on a distributed system. For the shared-memory model, the computational work is divided amongst worker threads, each of which has access to a pool of coordinates that it may work on, as well as to the shared parameters. This setup matches the system assumptions in \citet{niu2011hogwild,richtarik2012parallel,liu2013asynchronousCD}, and most modern multicore machines permit such an arrangement. On a distributed system, the parameter server~\citep{li2013parameter,dai2013petuum} broadcast the most recent parameter vector periodically to each worker and workers keep sending updates to the parameter vector after solving the subroutines corresponding to a randomly chosen parameter. In either settings, we do not wait fo slower workers or synchronize the parameters at any point of the algorithm, therefore many updates sent from the workers could be calculated based on a delayed parameter.

\begin{algorithm}[t]
	\caption{\algo: Asynchronous Parallel Block-Coordinate Frank-Wolfe (Distributed)}
	\label{alg:PFW}
	\begin{algorithmic}
		\STATE{\textsc{------------------------Server node---------------------}}
		\STATE {\bfseries Input:} An initial feasible $x^{(0)}$, mini-batch size $\tau$, number of workers $T$.
		
		\STATE{Broadcast $x^{(0)}$ to all workers.}
		\FOR{$k$ = 1,2,... ($k$ is the iteration number.)}
		\STATE{1. } Read from buffer until it has updates for $\tau$ disjoint blocks (overwrite in case of collision\footnotemark). Denote the index set by $S$.
		\STATE{2. } Set step size $\gamma=\frac{2n\tau}{\tau^2k+ 2n}$.
		\STATE{3. } Update $x^{(k)} = x^{(k-1)} + \gamma \sum_{i\in S}( s_{[i]} - x_{[i]}^{(k-1)})$.
		\STATE{4. } Broadcast $x^{(k)}$ (or just $x^{(k)} - x^{(k-1)}$) to workers.
				\IF{converged}
				\item Broadcast STOP signal to workers and break.
				\ENDIF
		\ENDFOR
		\STATE {\bfseries Output: } $x^{(k)}$.
		
		\STATE{\textsc{-----------------------Worker nodes---------------------}}
				\STATE{a. } Set $x$ to be $x^{(0)}$
				\WHILE {no STOP signal received}
				\IF {New update $x'$ is received}
				\STATE{b.} Update $x = x'$.
				\ENDIF
				\STATE{c.} Randomly choose $i\in [n]$.
				\STATE{d.} Calculate partial gradient $\nabla_{(i)}f(x)$ and solve \eqref{eq:exact_oracle}.
				\STATE{e.} Send $\{i,s_{(i)}\}$ to the server.
				\ENDWHILE
	\end{algorithmic}
\end{algorithm}
 \footnotetext{We bound the probability of collisions in Appendix~\ref{sec:distr}.}

The above scheme is made explicit by the pseudocode in Algorithm~\ref{alg:PFW}, following a server-worker terminology. The shared memory version of the pseudo-code is very similar, hence deferred to the Appendix.  The three most important questions pertaining to Algorithm~\ref{alg:PFW} are:
\begin{itemize}
	\setlength{\itemsep}{-1pt}
	\item Does it converge?
	\item If so, then how fast? And how much faster is it compared to \bcfw ($\tau=1$)?
	\item How do delayed updates affect the convergence?
\end{itemize}

We answer the first two questions in Section~\ref{sec:conv} and \ref{sec:speedup}. Specifically, we show \algo converges at the familiar $O(1/k)$ rate.  Our analysis reveals that the speed-up of \algo over \bcfw through parallelization is problem dependent. Intuitively, we show that the extent that mini-batching ($\tau > 1$) can speed up convergence depends on the average ``coupling'' of the objective function $f$ across different coordinate blocks. For example, we show that if $f$ has a block symmetric diagonally dominant Hessian, then \algo converges $\tau/2$ times faster.  We address the third question in Section~\ref{sec:delayed_gradient}, where we establish convergence results that depend only mildly in the ``expected'' delay $\kappa$. The bound is proportional to $\kappa$ when we allow the delay to grow unboundedly, and proportional to $\sqrt{\kappa}$ when the delay is bounded by a small $\kappa_{\max}$. 



\subsection{Main convergence results}\label{sec:conv}
Before stating the results, we need to define a few quantities. The first key quantity---also key to the analysis of several other FW methods---is the notion of \textbf{curvature}.  Since \algo updates a subset of coordinate blocks at a time, we define \emph{set curvature} for an index set $S \subseteq [n]$ as
\begin{align}
\label{eq:curv_S}
C_f^{(S)}:= \sup_{\begin{subarray}{l} x\in \cM, s_{(S)}\in \cM^{(S)},\\\gamma\in[0,1],\\y=x+\gamma(s_{[S]}-x_{[S]})\end{subarray}}
\frac{2}{\gamma^2} &\big(f(y)-f(x)- \\[-7mm]
&\hspace{2mm}\langle y_{(S)}-x_{(S)},\nabla_{(S)}f(x)\rangle\big) \nonumber
\end{align}
For index sets of size $\tau$, we define the \emph{expected set curvature} over a uniform choice of subsets as
\begin{equation}\label{eq:curv_tau}
C_f^{\tau} := \E_{S: |S|=\tau}[C_f^{(S)}] = {\textstyle {n\choose \tau}}^{-1}\nlsum_{S\subset [n], |S|=\tau}C_f^{(S)}.
\end{equation}
These curvature definitions are closely related to the global curvature constant $C_f$ of \citep{jaggi2013revisiting} and the coordinate curvature $C_f^{(i)}$ and product curvature $C_f^{\otimes}$ of \citep{lacoste2013block}. Lemma~\ref{lem:prop.curvature} makes this relation more precise.
\begin{lemma}[Curvature relations]\label{lem:prop.curvature}
	Suppose $S \subseteq [n]$ with cardinality $|S| = \tau$ and \( i \in S \). Then,
	\begin{enumerate}
		\item \quad$C_f^{(i)}\leq C_f^{(S)}\leq C_f$;\quad\quad 
		\item \quad$\frac{1}{n}C_f^{\otimes}=C_f^{1}\leq C_f^{\tau} \leq C_f^{n}=C_f$.
	\end{enumerate}
\end{lemma}
The way the average set curvature $C_f^{\tau}$ scales with $\tau$ is critical for bounding the amount of speedup we can expect over \bcfw; we provide a detailed analysis of this speedup in Section~\ref{sec:speedup}.

The next key object is an \textbf{approximate linear minimizer}. At iteration $k$, as in~\citet{jaggi2013revisiting,lacoste2013block}, we also low the core computational subroutine that solves \eqref{eq:exact_oracle} to yield an approximate minimizer $s_(i)$. The approximation is quantified by an additive constant $\delta \ge 0$ that for a minibatch 
$S\subset [n]$ of size $\tau$, the approximate solution $s_{(S)} :=\prod_{i\in S}s_{(i)}$ obeys \emph{in expectation} that
\begin{equation}\label{eq:approx_add}
\E\left[\langle s_{(S)},\nabla_{(S)}f^{(k)} \rangle -
\min_{s'\in \cM^{(S)}}\langle s',\nabla_{(S)}f^{(k)} \rangle\right] \leq  \frac{\delta \gamma_kC_f^{\tau}}{2}.
\end{equation}
%
%
%
where the expectation is taken over both the random coins in selecting $S$ and any other source of uncertainty in this oracle call during the entire history up to step $k$. \eqref{eq:approx_add} is strictly weaker than what is required in \citet{jaggi2013revisiting,lacoste2013block}, as we only need the approximation to hold in expectation. 
With definitions~\eqref{eq:curv_tau} and~\eqref{eq:approx_add} in hand, we are ready to state our first main convergence result.

\begin{theorem}[Primal Convergence]\label{thm:PFW_primal}
	Suppose we employ a linear minimizer that satisfies~\eqref{eq:approx_add} when solving $\tau$ subproblems (\ref{eq:exact_oracle}). Then, for each $k\geq 0$, the iterations in Algorithm~\ref{alg:PFW} and its line search variant (Steps 2b and 5 in Algorithm~\ref{alg:PFW}) obey
	\begin{equation*}
	\E[f(x^{(k)})] - f(x^*) \leq \frac{2n C}{\tau^2 k + 2n},
	\end{equation*}
	where the constant $C= nC_f^{\tau}(1+\delta)+f(x^{(0)})-f(x^*).$
\end{theorem}
At a first glance, the $n^2C_f^{\tau}$ term in the numerator might seem bizzare, but as we will see in the next section, $C_f^{\tau}$ can be as small as $O(\frac{\tau}{n^2})$. This is the scale of the constant one should keep in mind to compare the rate to other methods, e.g. coordinate descent.
Also note that so far this convergence result does not explicitly work for delayed updates, which we will analyze in Section~\ref{sec:delayed_gradient} separately via the approximation parameter $\delta$.

For FW methods, one can also easily obtain a convergence guarantee in an appropriate primal-dual sense. To this end, we introduce our version of the \textbf{surrogate duality gap}~\citep{jaggi2013revisiting}; we define this as 
\begin{align}
g(x) &= \max_{s\in\cM} \langle x-s,\nabla f(x) \rangle \label{eq:duality_gap}\\
&=\sum_{i=1}^{n} \max_{s_{(i)}\in \cM^{(i)}} \langle x_{(i)}-s_{(i)},\nabla_{(i)}f(x)\rangle =\sum_{i=1}^n g^{(i)}(x).\nonumber
\end{align}
To see why~\eqref{eq:duality_gap} is actually a duality gap, note that since $f$ is convex, the linearization $f(x)+\langle s-x, \nabla f(x)\rangle$ is always smaller than the function evaluated at any $s$, so that
$$g(x)\geq \langle x-x^*,\nabla f(x) \rangle\geq f(x)-f(x^*).$$
This duality gap is obtained for ``free'' in batch Frank-Wolfe, but not in \bcfw or \algo. Here, we only have an unbiased estimator $\hat{g}(x) = \frac{n}{|S|}\sum_{i\in S} g^{(i)}(x)$. As $\tau$ gets large, $\hat{g}(x)$ is close to $g(x)$ with high probability (McDiarmid's Inequality), and can  still be useful as a stopping criterion.

%

\begin{theorem}[Primal-Dual Convergence]\label{thm:PFW_primal-dual}
	Suppose we run Algorithm~\ref{alg:PFW} and its line search variant up to $K$ iterations, let $g_k :=\E g(x^{(k)})$ and $K\geq 1$, then there exists at least one $k^*\in [1,...,K]$ such that the expected surrogate duality gap satisfies
	$$
	g_{k^*}\leq \bar{g}_K \leq \frac{6nC}{\tau^2(K+1)},
	$$
	where $C$ is as in Theorem~\ref{thm:PFW_primal} and $\bar{g}_k$ is the weighted average $\frac{2}{K(K+1)}\sum_{k=1}^K kg_k$.
	%
\end{theorem}

{\bf Relation with FW and \bcfw:} The above convergence guarantees can be thought of as an interpolation between \bcfw and batch FW. If we take $\tau=1$, this gives exactly the convergence guarantee for \bcfw \citep[Theorem 2]{lacoste2013block} and if we take $\tau=n$, we can drop $f(x^{(0)})-f(x^*)$ from $C$ (with a small modification in the analysis) and it reduces to the classic batch guarantee as in \citep{jaggi2013revisiting}.

{\bf Dependence on initialization:} Unlike classic FW, the convergence rate for our method depends on the initialization. When $h_0:=f(x^{(0)})-f(x^*)\geq n C_f^\tau$ and $\tau^2<n$, the convergence is slower by a factor of $\frac{n}{\tau^2}$. The same concern was also raised in \citep{lacoste2013block} with $\tau=1$. We can actually remove the $f(x^{(0)})-f(x^*)$ from $C$ as long as we know that $h_0\leq n C_f^\tau$. By Lemma~\ref{lem:prop.curvature}, the expected set curvature $C_f^{\tau}$ increases with $\tau$, so the fast convergence region becomes larger when we increase $\tau$. In addition, if we pick $\tau^2>n$, the rate of convergence is not  affected by initialization anymore.


{\bf Speedup:} The careful reader may have noticed the $n^2C_f^{\tau}$ term in the numerator. This is undesirable as $n$ can be large (for instance, in structural SVM $n$ is the total number of data points). The saving grace in \bcfw is that when $\tau=1$,  $C_f^{\tau}$ is as small as $O(n^{-2})$ (see \citep[Lemmas A1~and~A2]{lacoste2013block}), and it is easy to check that the dependence in $n$ is the same even for $\tau>1$. What really matters is how much speedup one can achieve over \bcfw, and this speedup critically relies on how $C_f^{\tau}$ depends on $\tau$. Analyzing this dependence will be our main focus in the next section.

\subsection{Effect of parallelism / mini-batching}\label{sec:speedup}
To understand when mini-batching is meaningful and to quantify its speedup, we take a more careful look at the expected set curvature $C_f^{\tau}$ in this section. In particular, we analyze and present a set of insightful conditions  that govern its relationship with $\tau$. The key idea is to roughly quantify how strongly different coordinate blocks interact with each other.

To begin, assume that there exists a positive semidefinite matrix $H$ such that for any $x,y\in\cM$
\begin{equation}\label{eq:H}
f(y) \leq f(x) + \langle y-x, \nabla f(x)\rangle +  \frac{1}{2}(y-x)^TH (y-x).
\end{equation}
The matrix $H$ may be viewed as a generalization of the gradient's Lipschitz constant (a scalar) to a matrix. For quadratic functions $f(x)=\frac{1}{2}x^TQ x +c^Tx$, we can take $H=Q$. For twice differentiable functions, we can choose
$
H \in \{K \;|\; K \succeq \nabla^2 f(x), \;\;\forall x\in \cM\}.
$


Since $x=[x_1, ...,x_n]$ (we write $x_i$ instead of $x_{(i)}$ for brevity), we  separate  $H$ into $n\times n$ blocks; so $H_{ij}$ represents the block corresponding to $x_i$ and $x_j$ such that we can take the product $x_i^TH_{ij}x_j$. Now, we define a \emph{boundedness} parameter $B_{i}$ for every $i$, and an \emph{incoherence condition} with parameter $\mu_{ij}$ for every block coordinate pair $\cM_i, \cM_j$ such that
\begin{align*}
&B_{i} =\sup_{x_i\in \cM_i} x_i^TH_{ii}x_i, &&
\mu_{ij} =\sup_{x_i\in \cM_i, x_j\in \cM_j} x_i^TH_{ij}x_j, \\
&B =\E_{i\sim \mathrm{Unif}([n])} B_{i}, &&
\mu =\E_{(i,j)\sim \mathrm{Unif(\{(i,j)\in [n]^2, i\neq j\})}} \mu_{ij}. 
\end{align*}
Then, using these quantities, we obtain the following bound on the expected set-curvature.
\begin{theorem}\label{thm:Cf_master_lemma}
	If problem~\eqref{eq:1} obeys $B$-expected boundedness and $\mu$-expected incoherence. Then,
	\begin{equation}\label{eq:Cf^tau_tightbound}
	C_f^{\tau} \leq 4(\tau B +\tau(\tau-1) \mu)\qquad\text{for any} \quad \tau=1,...,n.
	\end{equation}
\end{theorem}

It is clear that when the incoherence term $\mu$ is large, the expected set curvature $C_f^\tau$ is proportional to $\tau^2$, and when $\mu$ is close to 0, then $C_f^{\tau}$ is proportional to $\tau$. In other words, when the interaction between  coordinates block is small, one would gain from parallelizing the block-coordinate Frank-Wolfe.
This is analogous to the situation in parallel coordinate descent \citep{richtarik2012parallel,liu2013asynchronousCD} and we will compare the rate of convergence explicitly with them in the next section.

\begin{remark}
	Let us form a matrix $M$ with $B_i$ on the diagonal and $\mu_{ij}$ on the off-diagonal. If $M$ is \emph{symmetric diagonally dominant} (SDD), i.e., the sum of absolute off-diagonal entries in each row is no greater than the diagonal entry, then $C_f^{\tau}$ is proportional to $\tau$. 
\end{remark}

The above result depends on the parameters $B$ and $\mu$. We now derive specific instances of the above results for the structural SVM and Group Fused Lasso. For the structural SVM, a simple generalization of \citep[Lemmas A.1, A.2]{lacoste2013block} shows that in the worst case, using $\tau>1$ offers no gain at all. Fortunately, if we are willing to consider a more specific problem and consider the average case instead, using larger $\tau$ does make the algorithm converge faster (and this is the case according to our experiments).


\begin{example}[{\bf Structural SVM for multi-label classification (with random data)}]\label{eg:structSVM_avg}
	We describe the application to structural SVMs in detail in Section~\ref{structSVM} (please see this section for details on notation). Here, we describe the convergence rate for this application. According to \citep{yu2009learning}, the compatibility function $\phi(x,y)$ for multiclass classification will be $[0,...,0,x^T,0,...0]^T/\lambda n$ where the only nonzero block that we fill with the feature vector is the $(y)$th block. So $\psi_i(x_i,j)=\phi(x_i,y_i) - \phi(x_i,j)$ looks like $[0,...,0,$ $x_i^T,0$ $,...0,$ $-x_i^T,0,$ $...0]^T/\lambda n$. This already ensures that
	$B = \frac{2}{n^2\lambda}$ provided $x_i$ lie on a unit sphere. Suppose we have $K$ classes and each class has a unique feature vector drawn randomly from a unit sphere in $\R^d$; furthermore, for simplicity assume we always draw $\tau<K$ data points with $\tau$ distinct labels\footnote{This is an oversimplification but it offers a rough rule-of-thumb. In practice, $C_f^{\tau}$ should be in the same ballpark as our estimate here.}
	$
	\mu \leq \sqrt{\frac{C\log d}{ d}}\frac{2}{n^2\lambda},
	$
	for some constant $C$. In addition, if $d\geq \tau^2\sqrt{C\log d}$, then with high probability
	$$
	C_f^{\tau} \leq \frac{2\tau + 2\tau^2\sqrt{\frac{C\log d}{ d}}}{n^2\lambda} \leq \frac{C\tau}{n^2\lambda},
	$$
	which yields a convergence rate $O(\frac{R^2}{\lambda \tau k})$, where $R$ $:=$\\ $\max_{i\in[n], y\in\cY_i}$ $\|\psi_i(y)\|_2$ using notation from Lemmas~A.1~and~A.2 of \citep{lacoste2013block}.
	
	This analysis suggests that a good rule-of-thumb is that we should choose $\tau$ to be at most the number of categories for the classification.
	If each class is a mixture of random draws from the unit sphere, then we can choose $\tau$ to be the underlying number of mixture components.
\end{example}

\begin{example}[{\bf Group Fused Lasso}]\label{eg:group_fused_lasso}
	The Group Fused Lasso aims to solve (typically for $q=2$)
	\begin{align}
	\label{eq:4}
	\min_{X}\quad\tfrac{1}{2}\|X-Y\|_F^2 + \lambda\|XD\|_{1,q},\qquad q >1,
	\end{align}
	where $X,Y\in \R^{d\times n}$, and column $y_t$ of $Y$ is an observed noisy $d$-dimensional feature vector at time $1\le t \le n$. The matrix $D\in \R^{n\times (n-1)}$ is the differencing matrix that takes the difference of feature vectors at adjacent time points (columns). The formulation aims to filter the trend that has some piecewise constant structures.
	The dual to~\eqref{eq:4} is
	\begin{align*}
	\max_{U} \;&-\frac{1}{2}\|UD^T\|_F^2 + \tr U D^TY^T\\
	\text{s.t.}\quad &\|U_{:,t}\|_{p} \leq \lambda, \; \forall t=1,...,n-1,
	\end{align*}
	where $p$ is conjugate to $q$, i.e., $1/p+1/q=1$. This block-constrained problem fits our structure~\eqref{eq:1}. For this problem, we find that $B \leq 2\lambda^2 d$ and $\mu\leq \lambda^2 d$, which yields the bound
	$$
	C_f^{\tau} \leq 4\tau \lambda^2d.
	$$
	Consequently, the rate of convergence becomes $O(\frac{n^2 \lambda^2 d }{\tau k})$. In this case, batch FW will have a better rate of convergence than \bcfw \footnote{Observe that $C_f^{\tau}$ does not have an $n^2$ term in the denominator to cancel out the numerator. This is because the objective function is not appropriately scaled with $n$ like it does in the structural SVM formulation.}.
\end{example}


\subsection{Convergence with delayed updates}\label{sec:delayed_gradient}
Due to the delays in communication, it happens all the time that some updates pushed back by workers are calculated based on delayed parameters that we broadcast earlier. Dropping these updates or enforcing synchronization will create a huge system overhead especially when the size of the minibatch is small. Ideally, we want to just accept the delayed updates as if they are correct, and broadcast new parameters to workers without locking the updates. The question is, does it actually work? In this section, we model the delay from every update to be iid from an unknown distribution. Under weak assumptions, we show that the effect of delayed updates can be treated as a form of approximate oracle evaluation as in \eqref{eq:approx_add} with some specific constant $\delta$ that depends on the expected delay $\kappa$ and the maximum delay parameter $\kappa_{\max}$ (when exists), therefore establishing that the convergence results in the previous section remains valid for this variant. The results will also depend on the following diameter and gradient Lipschitz constant for a norm $\|\cdot\|$
\begin{align*}
D_{\|\cdot\|}^{(S)}  &= \sup_{x,y\in \cM^{(S)}} \|x-y\|,\\
L_{\|\cdot\|}^{(S)}  &=  \sup_{\begin{subarray}{l} x,y\in \cM, y= x+s\\\|s\|\leq \gamma,\\s \in \mathrm{span}(\cM^{(S)})\end{subarray}} \frac{1}{\gamma^2} (f(y) - f(x) - \langle y-x, \nabla f(x) \rangle),\\
D_{\|\cdot\|}^\tau  &= \max_{S\subset [n] \big| |S|= m}  D_{\|\cdot\|}^{(S)}, \text{ and }
	L_{\|\cdot\|}^\tau  = \max_{S\subset [n] | |S|= m} L_{\|\cdot\|}^{(S)}.
\end{align*}

\begin{theorem}[Delayed Updates as Approximate Oracle]\label{thm:delay}
For each norm $\|\cdot\|$ of choice, let $D_{\|\cdot\|}^\tau$ and $L_{\|\cdot\|}^\tau$ be defined above. Let the a random variable of delay be $\varkappa$ and let $\kappa := \E\varkappa $ be the expected delay from any worker, 
 moreover, assume that the algorithm drops any updates with delay greater than $k/2$ at iteration $k$.  Then for the version of the algorithm without line-search, the delayed oracle will produce $s \in \cM^{(S)}$ such that \eqref{eq:approx_add} holds with 
\begin{equation}\label{eq:thm_delay_bound1}
\delta = \frac{4\kappa \tau L_{\|\cdot\|}^{1}D_{\|\cdot\|}^1D_{\|\cdot\|}^\tau}{C_f^\tau}.
\end{equation}
Furthermore, if we assume that there is a $\kappa_{\max}$ such that $\P(\varkappa \leq \kappa_{\max}) =1 $ for all $k$, then \eqref{eq:approx_add} holds with $\delta = c_{n,\tau\kappa_{\max}}\frac{4\tau L_{\|\cdot\|}^{1}D_{\|\cdot\|}^1 \E D_{\|\cdot\|}^{\varkappa\tau}}{C_f^\tau}$ where $c_{n,\tau\kappa_{\max}} =$
\begin{equation}\label{eq:thm_delay_bound2}
\begin{cases}
\frac{3\log n}{\log(n/(\tau\kappa_{\max}))} & \text{ if } \kappa_{\max} \tau < n/\log n,\\
O(\log n) & 
\text{ if } \kappa_{\max} \tau  = O(n\log n),\\
\frac{(1+o(1))\tau\kappa_{\max}}{n} & \text{ if } \kappa_{\max} \tau  \gg  n\log n.
\end{cases}
\end{equation}
\end{theorem}
The results above imply that \algo (without line-search) converges in both primal optimality and in duality gap according to Theorem~\ref{thm:PFW_primal}~and~\ref{thm:PFW_primal-dual}.

Note that \eqref{eq:thm_delay_bound1} depends on the expected delay rather than the maximum delay and as $k\rightarrow \infty$ we allow the maximum delay to grow unboundedly. This allows the system to automatically deal with heavy-tailed delay distribution and sporadic stragglers. When we do have small bounded delay, we produce stronger bounds \eqref{eq:thm_delay_bound2} with a multiplier that is either a constant (when $\tau\kappa_{\max} = O(n^{1-\epsilon})$ for any $\epsilon>0$), proportional to $\log n$ (when $\tau\kappa\leq n$) or proportional to $\frac{\tau\kappa_{\max}}{n}$ (when $\tau\kappa$ is large). The whole expression often has sublinear dependency in the expected delay $\kappa$. To be more precise, when $\|\cdot\|$ is Euclidean norm, $\E D_{\|\cdot\|}^{\varkappa\tau} \leq \sqrt{\E \varkappa} D_{\|\cdot\|}^{\tau}$ by Jensen's inequality. Therefore in this case the bound is essentially proportional to $\sqrt{\kappa}$. This is strictly better than \citet{niu2011hogwild} which has quadratic dependency in $\kappa_{\max}$ and \citet{liu2013asynchronousCD} which has exponential dependency in $\kappa_{\max}$. Our mild $\kappa_{\max}$ dependency for the cases $\tau\kappa_{\max} > n$ suggests that the \eqref{eq:thm_delay_bound2} remains proportional to $\sqrt{\kappa}$ even when we allow the maximum delay parameter to be as large as $\frac{n}{\tau}$ or larger without significantly affecting the convergence. Note that this allows some workers to be delayed for several data passes.

Observe that when $\tau=1$, where the results reduces to a lock-free variant for \bcfw, $\delta$ becomes proportional to $\frac{L^1_\|\cdot\|[D^1_\|\cdot\|]^2}{C_f^1}$. This is always greater than $1$ (see e.g., \citep[Appendix D]{jaggi2013revisiting}) but due to the flexibility of choosing the norm, this quantity corresponding to the most favorable norm is typically a small constant. For example, when $f$ is a quadratic function, we show that $C_f^1 = L^1_\|\cdot\|[D^1_\|\cdot\|]^2$ (see Appendix~\ref{sec:lipschitz_curvature}). When $\tau > 1$, $\frac{\tau L_{\|\cdot\|}^{1}D_{\|\cdot\|}^1D_{\|\cdot\|}^\tau}{C_f^\tau}$ is often $O(\sqrt{\tau})$ for an appropriately chosen norm. Therefore, \eqref{eq:thm_delay_bound1} and \eqref{eq:thm_delay_bound2} are roughly in the order of $O(\kappa\sqrt{\tau})$ and $O(\sqrt{\kappa\tau})$ respectively\footnote{For details, see our discussion in Appendix~\ref{sec:lipschitz_curvature}}.

Lastly, we remark that $\kappa$ and $\tau$ are not independent. When we increase $\tau$, we update the parameters less frequently and $\kappa$ gets smaller. In a real distributed system, with constant throughput in terms of number of oracle solves per second from all workers. If the average delay is a fixed number in clock time specified by communication time. Then $\tau\kappa$ is roughly a constant regardless how $\tau$ is chosen.

\section{Experiments}
\label{sec:exp}

In this section, we experimentally demonstrate the performance gains of the three key features of our algorithm: minibatches of data, parallel workers, and asynchronous updates.

\subsection{Minibatches of Data}\label{sec:minibatch}
We conduct simulations to study the effect of mini-batch size $\tau$, where larger $\tau$ implies greater degrees of parallelism as each worker can solve one or more subproblems in a mini-batch. In our simulation for structural SVM we use sequence labeling task on a subset of the OCR dataset \citep{koller2003max} $(n=6251,d=4082)$. The subproblem can be solved using the Viterbi algorithm. The speedup on this dataset is shown in Figure~\ref{fig:speedup_simulation}(a). For this dataset, we use $\lambda=1$ with weighted averaging and line-search throughout. We measure the speedup for a particular $\tau>1$ in terms of the number of epochs (Algorithm~\ref{alg:PFW}) required to converge relative to $\tau=1$, which corresponds to \bcfw. Figure~\ref{fig:speedup_simulation}(a) shows that \algo achieves linear speedup for mini-batch size up to $\tau \approx 50$. Further speedup is sensitive to the convergence criteria, where more stringent thresholds lead to lower speed-ups. This is because large mini-batch sizes introduce errors, which reduces progress per update, and is consistent with existing work on the effect of parameter staleness on convergence~\citep{ho13,essp14}. This suggests that it might be possible to use more workers initially for a large speedup and reduce parallelism as the algorithm approaches the optimum.

In our simulation for Group Fused Lasso, we generate a piecewise constant dataset of size ($n=100, d=10$, in Eq.~\ref{eg:group_fused_lasso}) with Gaussian noise. We use $\lambda=0.01$ and a primal suboptimality threshold as our convergence criterion. At each iteration, we solve $\tau$ subproblems (i.e. the mini-batch size). Figure \ref{fig:speedup_simulation}(b) shows the speed-up over $\tau=1$ (\bcfw). Similar to the structural SVM, the speedup is almost perfect for small $\tau$ ($\tau \le 55$) but tapers off for large $\tau$ to varying degrees depending on the convergence thresholds.


\begin{figure}[tb]	
	\centering
	\begin{subfigure}[t]{0.42\textwidth}
		\centering
		\includegraphics[width=\textwidth]{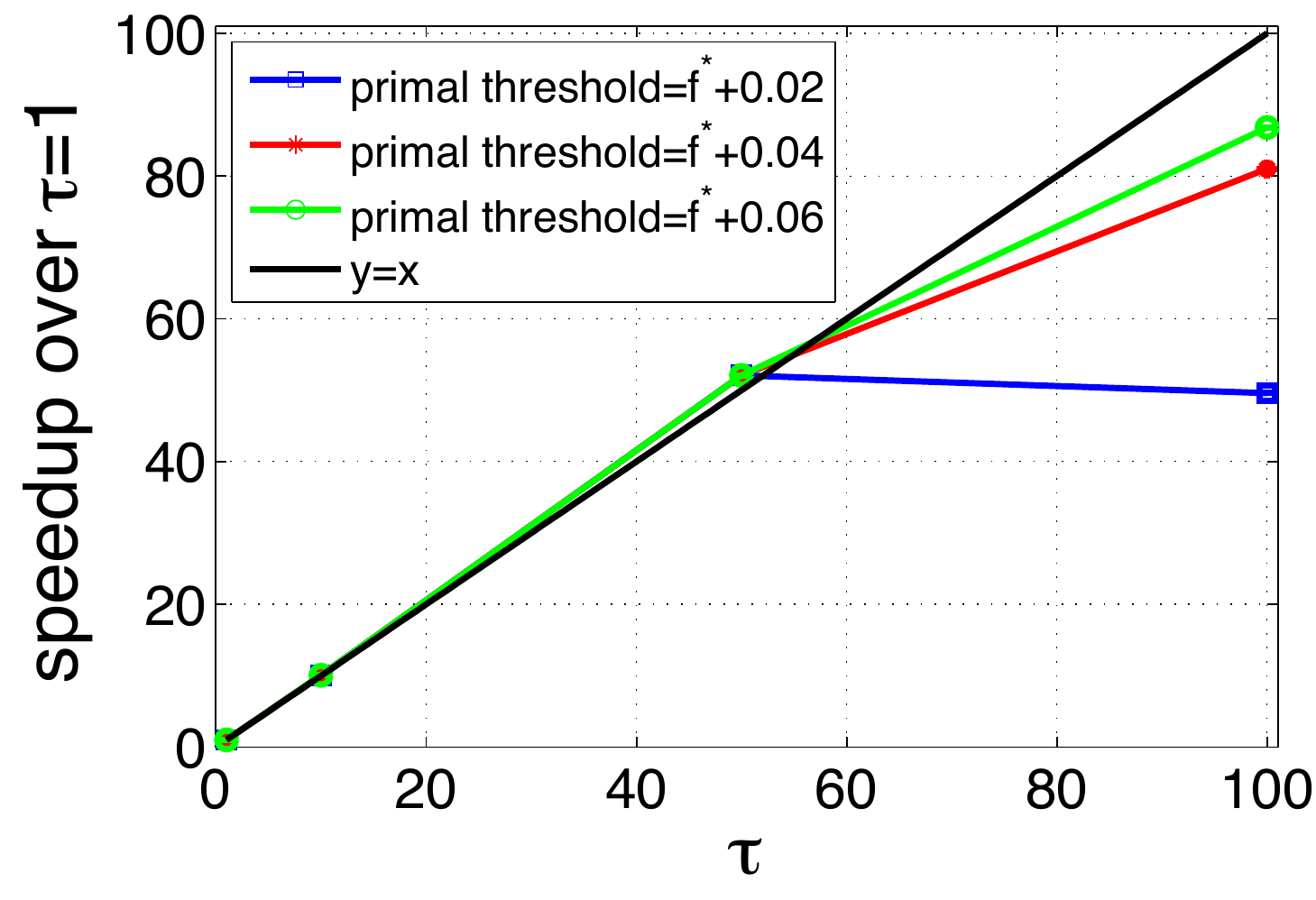}
		 		\caption{\small Structural SVM on OCR dataset (n=6251)}
		\label{fig:speedup_sim_ocr}
	\end{subfigure}
	\begin{subfigure}[t]{0.42\textwidth}
		\centering
		\includegraphics[width=\textwidth]{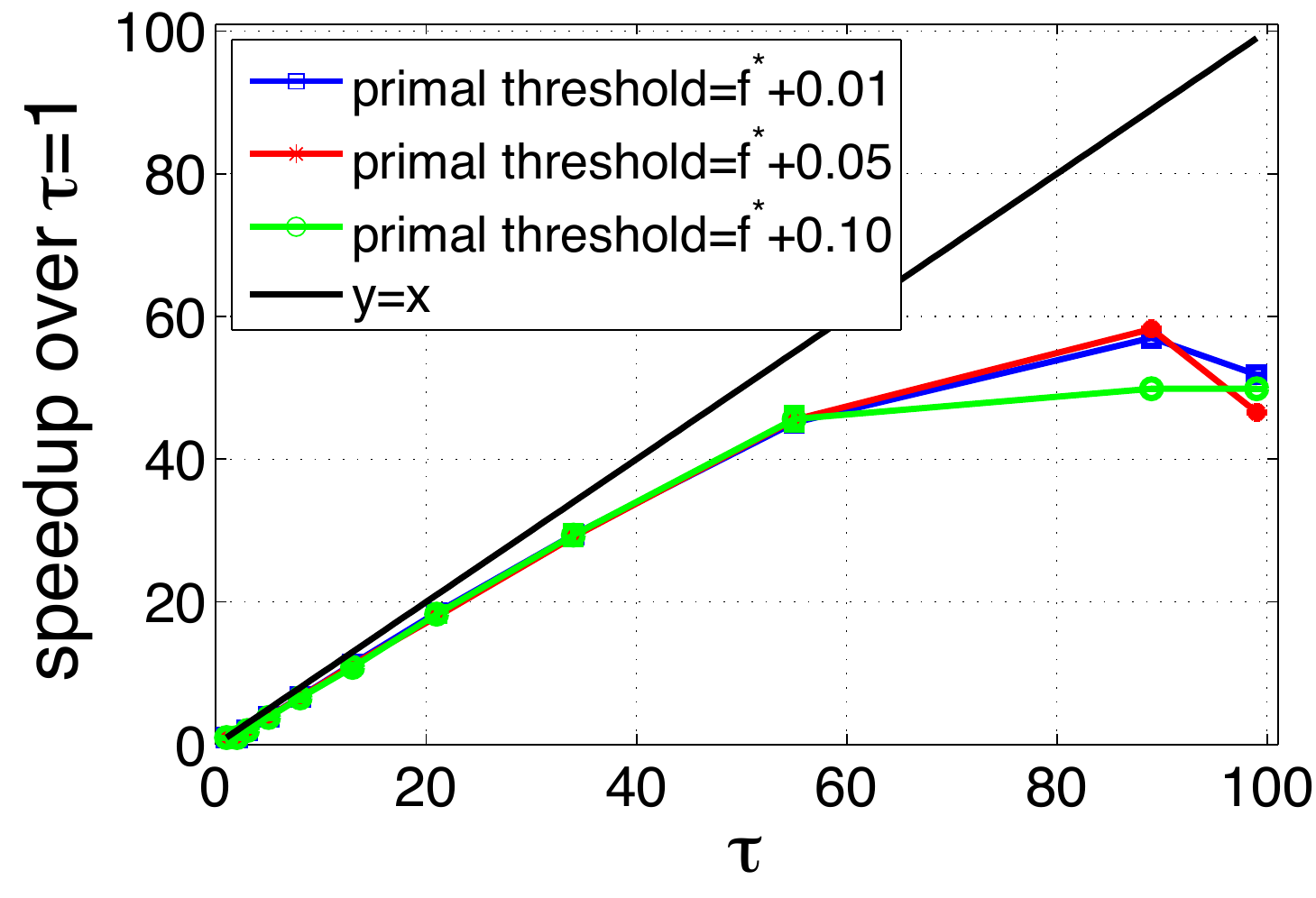}
		 		\caption{\small Group Fused Lasso (n=100)}
		\label{fig:speedup_sim_fused}
	\end{subfigure}
	\caption{\small Performance improvement with $\tau$ for (a) Structual SVM on the OCR dataset~\citep{koller2003max} and (b) Group Fused Lasso on a synthetic dataset. $f^*$ denotes primal optimum.}
	\label{fig:speedup_simulation}
	\vspace{-0.5cm}
\end{figure}

\subsection{Shared Memory Parallel Workers}
We implement \algo for the structural SVM in a multicore shared-memory system using the full OCR dataset $(n=6877)$. All shared-memory experiments were implemented in C++ and conducted on a 16-core machine with Intel(R) Xeon(R) CPU E5-2450 2.10GHz processors and 128G RAM. We first fix the number of workers at $T=8$ and vary the mini-batch size $\tau$. Figure \ref{fig:primal_shared_memory}(a) shows the absolute convergence (i.e. the convergence per second). We note that \algo outperforms single-threaded \bcfw under all investigated $\tau$, showing the efficacy of parallelization. Within \algo, convergence improves with increasing mini-batch sizes up to $\tau = 3T$, but worsens when $\tau=5T$ as the error from the large mini-batch size dominates additional computation. The optimal $\tau$ for a given number of workers ($T$) depends on both the dataset (how ``coupled'' are the coordinates) and also system implementations (how costly is the synchronization as the system scales).

Since speedup for a given $T$ depends on $\tau$, we search for the optimal $\tau$ across multiples of $T$ to find the best speedup for each $T$. Figure \ref{fig:primal_shared_memory}(b) shows faster convergence of \algo over \bcfw ($T=1$) when $T>1$ workers are available. It is important to note that the x-axis is \emph{wall-clock time} rather than the number of epochs.

Figure \ref{fig:primal_shared_memory}(c) shows the speedup with varying $T$. \algo achieves near-linear speed up for smaller $T$. The speed-up curve tapers off for larger $T$ for two reasons: (1) Large $T$ incurs higher system overheads, and thus needs larger $\tau$ to utilize CPU efficiently; (2) Larger $\tau$ incurs errors as shown in Fig.~\ref{fig:speedup_simulation}(a). 
If the subproblems were more time-consuming to solve, the affect of system overhead would be reduced. We simulate harder subproblems by simply solving them $m\sim $Uniform$(5,15)$ times instead of just once. The speedup is nearly perfect as shown in Figure \ref{fig:primal_shared_memory}(d).
Again, we observe that a more generous convergence threshold produces higher speedup, suggesting that resource scheduling could be useful (e.g., allocate more CPUs initially and fewer as algorithm converges).



\begin{figure*}[t]	
	\centering
	\begin{subfigure}[t]{0.42\textwidth}
		\centering
		\includegraphics[width=\textwidth]{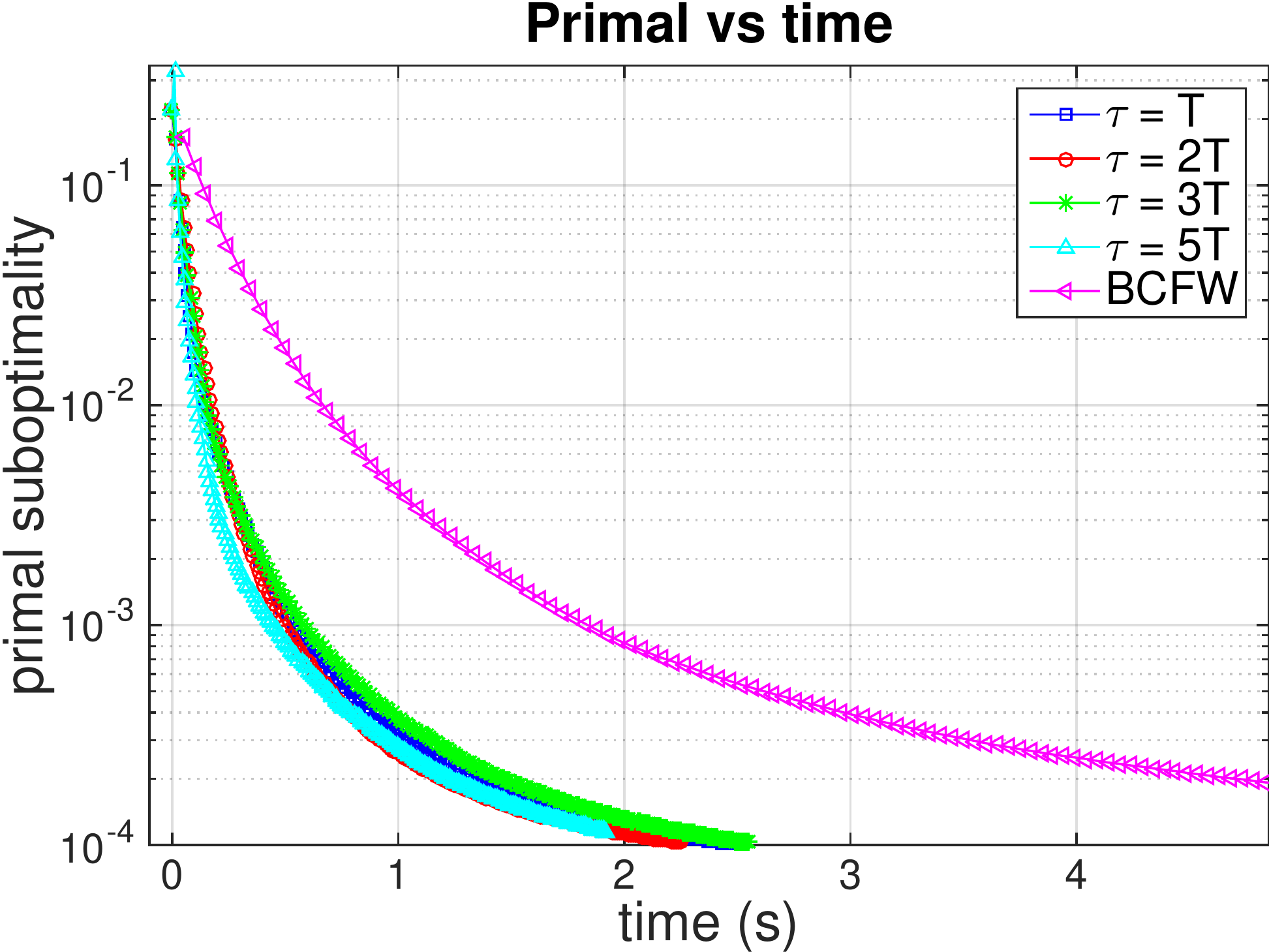}
		
		\vspace{-2mm}
		 		\caption{\small }
		\label{fig:shared_memory01}
	\end{subfigure}
	\begin{subfigure}[t]{0.42\textwidth}
		\centering
		\includegraphics[width=\textwidth]{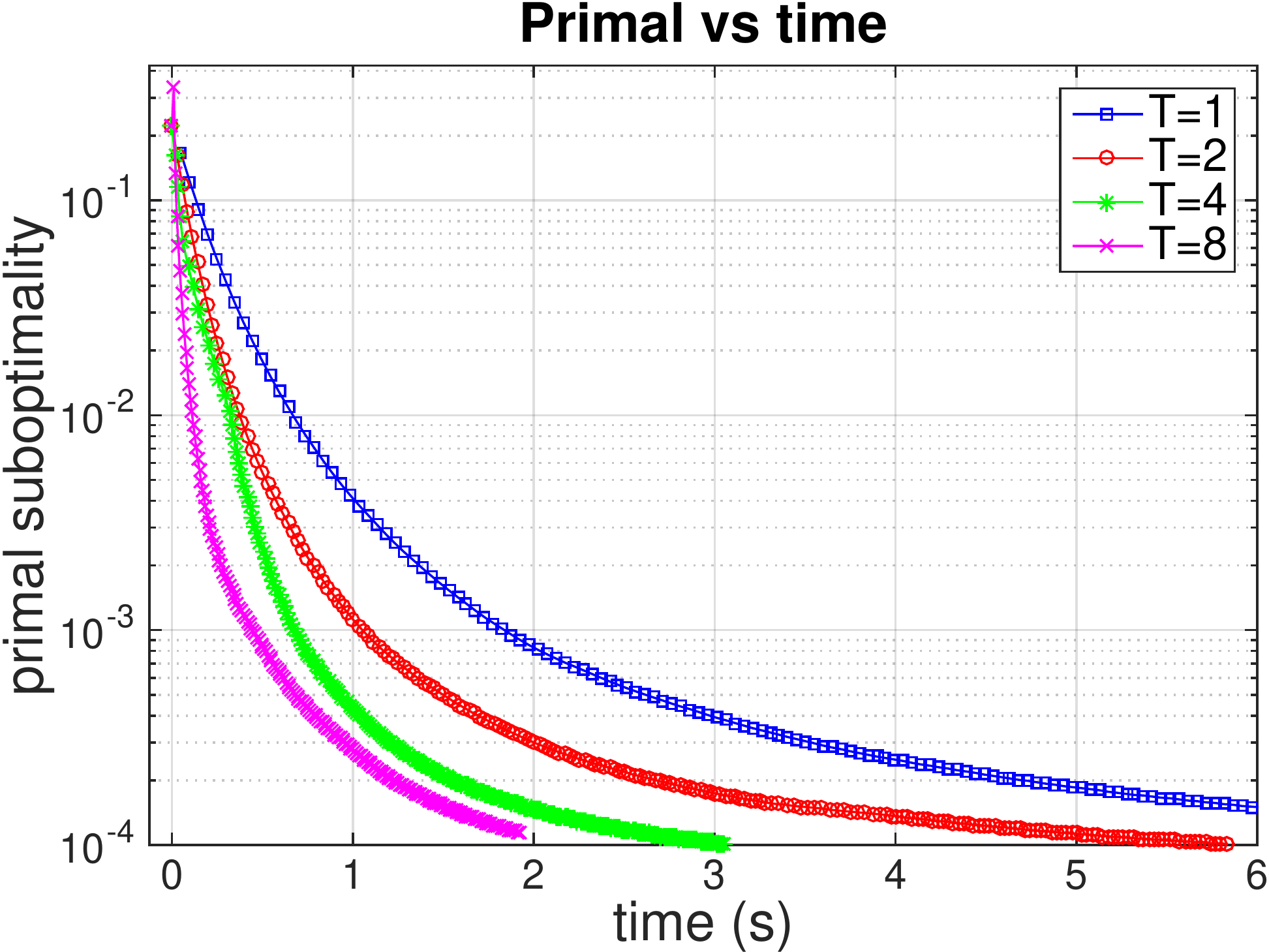}
	
	\vspace{-2mm}
		 		\caption{\small }
		\label{fig:shared_memory03}
	\end{subfigure}\\
	\begin{subfigure}[t]{0.42\textwidth}
		\centering
		\includegraphics[width=\textwidth,height= 0.75\textwidth]{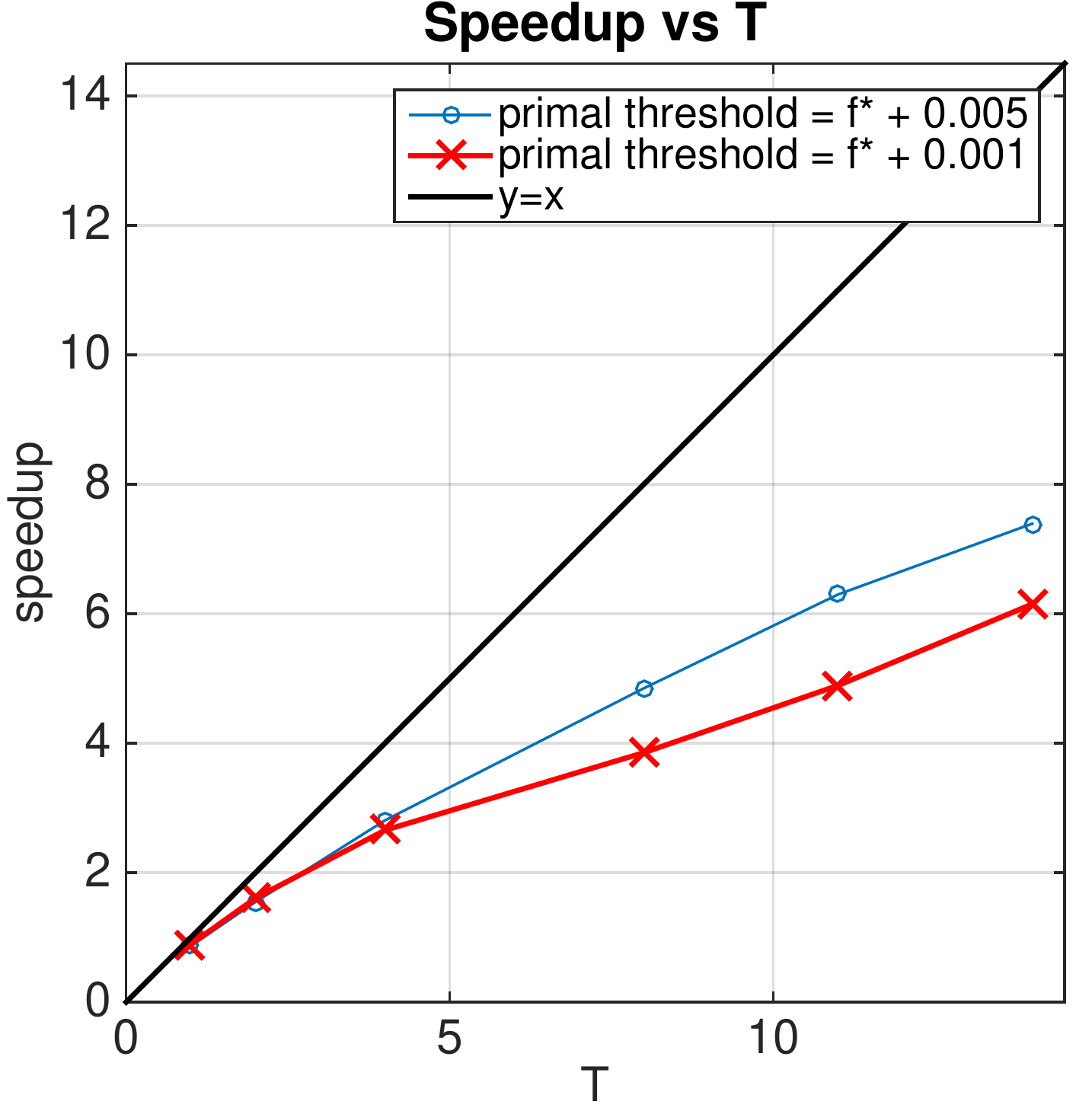}
		
		\vspace{-2mm}
		 		\caption{\small }
		\label{fig:shared_memory02}
	\end{subfigure}	
	\begin{subfigure}[t]{0.42\textwidth}
		\centering
		\includegraphics[width=\textwidth,height = 0.75\textwidth]{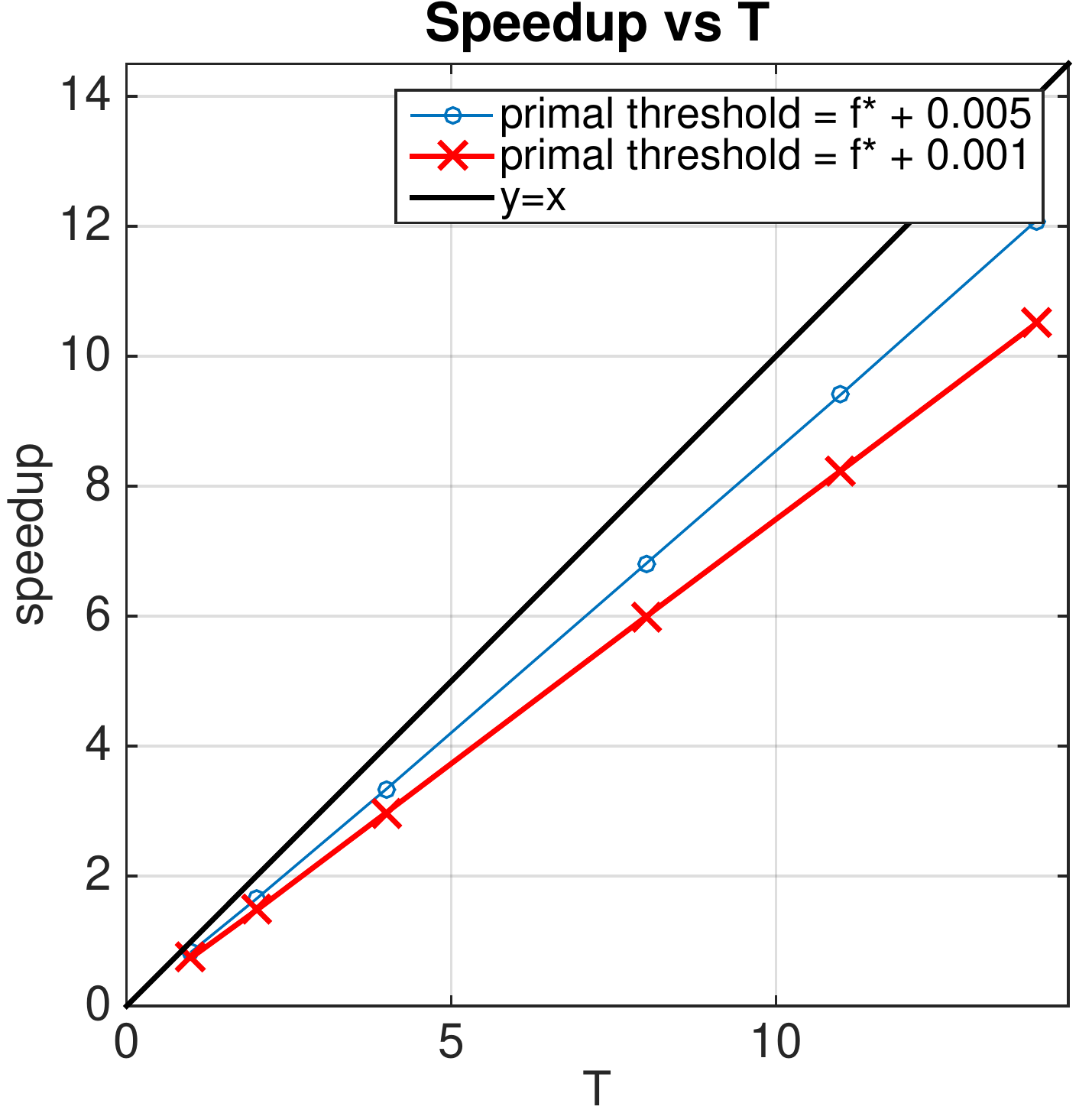}
		
				\vspace{-2mm}
		 		\caption{\small }
		\label{fig:shared_memory04}
	\end{subfigure}	
	\vspace{-0.5cm}
	\caption{\small From left: \textbf{(a)} Primal suboptimality vs wall-clock time using 8 workers ($T=8$) and various mini-batch sizes $\tau$.  \textbf{(b)} Primal suboptimality vs wall-clock time for varying $T$ with best $\tau$ chosen for each $T$ separately. \textbf{(c)} Speedup via parallelization with the best $\tau$ chosen among multiples of $T$ ($T, 2T,...$) for each $T$. \textbf{(d)} The same with longer subproblems.}
	\label{fig:primal_shared_memory}
	\vspace{-0.5cm}
\end{figure*}

\vspace{-1mm}
\subsection{Performance gain with asynchronous updates}
We compare \algo with a synchronous version of the algorithm (\spalgo) where the server assigns $\tau/T$ subproblems to each worker, then waits for and accumulates the solutions before proceeding to the next iteration. We simulate workers of varying slow-downs in our shared-memory setup by assigning a \emph{return probability} $p_i\in (0,1]$ to each worker $w_i$. After solving each subproblem, worker $w_i$ reports the solution to the server with probability $p_i$. Thus a  worker with $p_i = 0.8$ will drop 20\% of the updates on average corresponding to $20\%$ slow-down. 

%
%

We use $T=14$ workers for the experiments in this section. We first simulate the scenario with just one straggler with return probability $p\in (0,1]$ while the other workers run at full speed $(p=1)$. Figure \ref{fig:async_vs_sync}(a) shows that the average time per effective datapass (over 20 passes and 5 runs) of \algo stays almost unchanged with slowdown factor $1/p$ of the straggler, whereas it increases linearly for \spalgo. This is because \algo relies on the average available worker processing power, while \spalgo is only as fast as the slowest worker.

Next, we simulate a heterogeneous environment where the workers have varying speeds. While varying a parameter $\theta \in [0,1]$, we set $p_i = \theta + i/T$ for $i=1,\cdots,T$. Figure \ref{fig:async_vs_sync}(b) shows that \algo slows down by only a factor of 1.4 compared to the no-straggler case. Assuming that the server and worker each takes about half the (wall-clock) time on average per epoch, we would expect the run time to increase by 50\% if average worker speed halves, which is the case if $\theta = 0$ (i.e., $\frac{1}{\theta}\rightarrow \infty$).
Therefore a factor of 1.4 is reasonable. The performance of \spalgo is almost identical to that in the previous experiment as its speed is determined by the slowest worker. Thus our experiments show that \algo is robust to stragglers and system heterogeneity.

\begin{figure}[t]	
	\centering
	\includegraphics[width=0.9\textwidth]{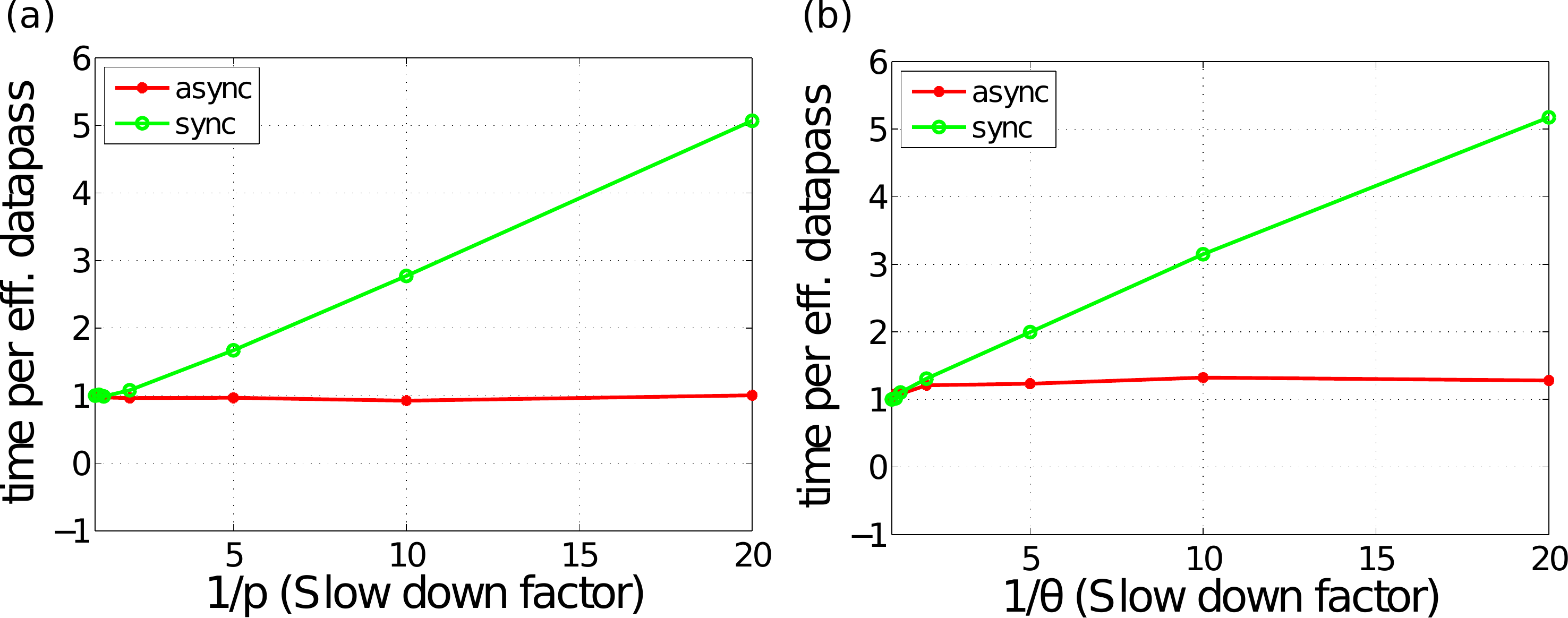}
	\caption{\small Average time per effective data pass in asynchronous and synchronous modes for two cases: one worker is slow with return probability $p$ (left); workers have return probabilities $(p_i$s$)$ uniformly in $[\theta,1]$ (right). Times normalized separately for \algo, \spalgo w.r.t.\ to the setup where workers run at full speed.} \label{fig:async_vs_sync}
	\vspace{-0.5cm}
\end{figure}

\subsection{Convergence under unbounded heavy-tailed delay}
In this section, we illustrate the mild effect of delay on convergence by randomly drawing an independent delay variable for each worker. For simplicity, we use $\tau = 1$ (\bcfw) on the same group fused lasso problem as in Section~\ref{sec:minibatch}. We sample $\varkappa$ using either a Poisson distribution or a heavy-tailed Pareto distribution (round to the nearest integer). The Pareto distribution is chosen with shape parameter $\alpha=2$ and scale parameter $x_m=\kappa/2$ such that $\E \varkappa =\kappa$ and $\Var\varkappa = \infty$. During the experiment, at iteration $k$, any updates that were based on a delay greater than $k/2$ are dropped (as our theory demanded). The results are shown in Figure~\ref{fig:delay}. Observe that for both cases, the impact of the delay is rather mild. With expected delay up to $20$, the algorithm only takes fewer than twice as many iterations to converge.
\begin{figure}[t]	
	\centering
	\includegraphics[width=0.45\textwidth]{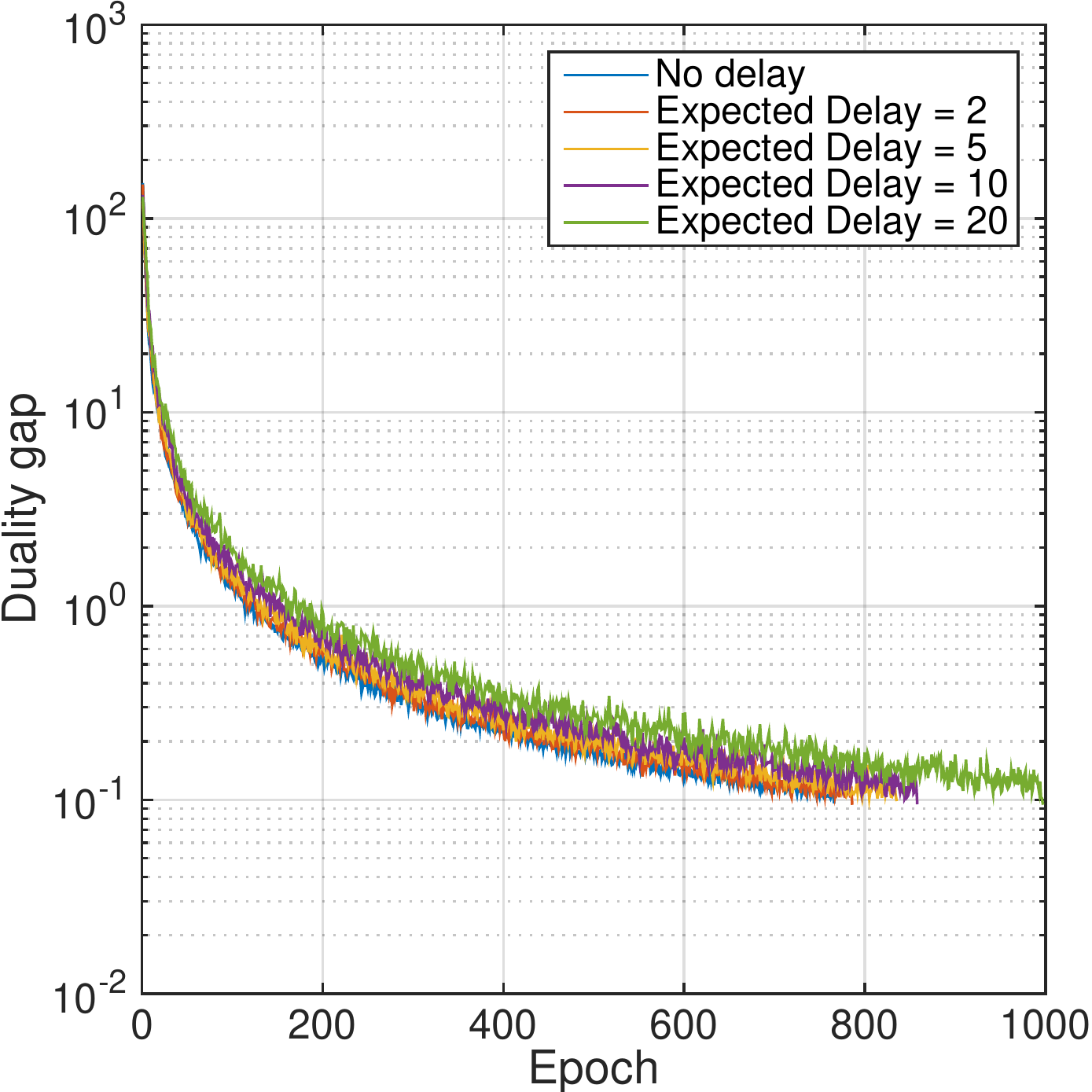}
	\includegraphics[width=0.45\textwidth]{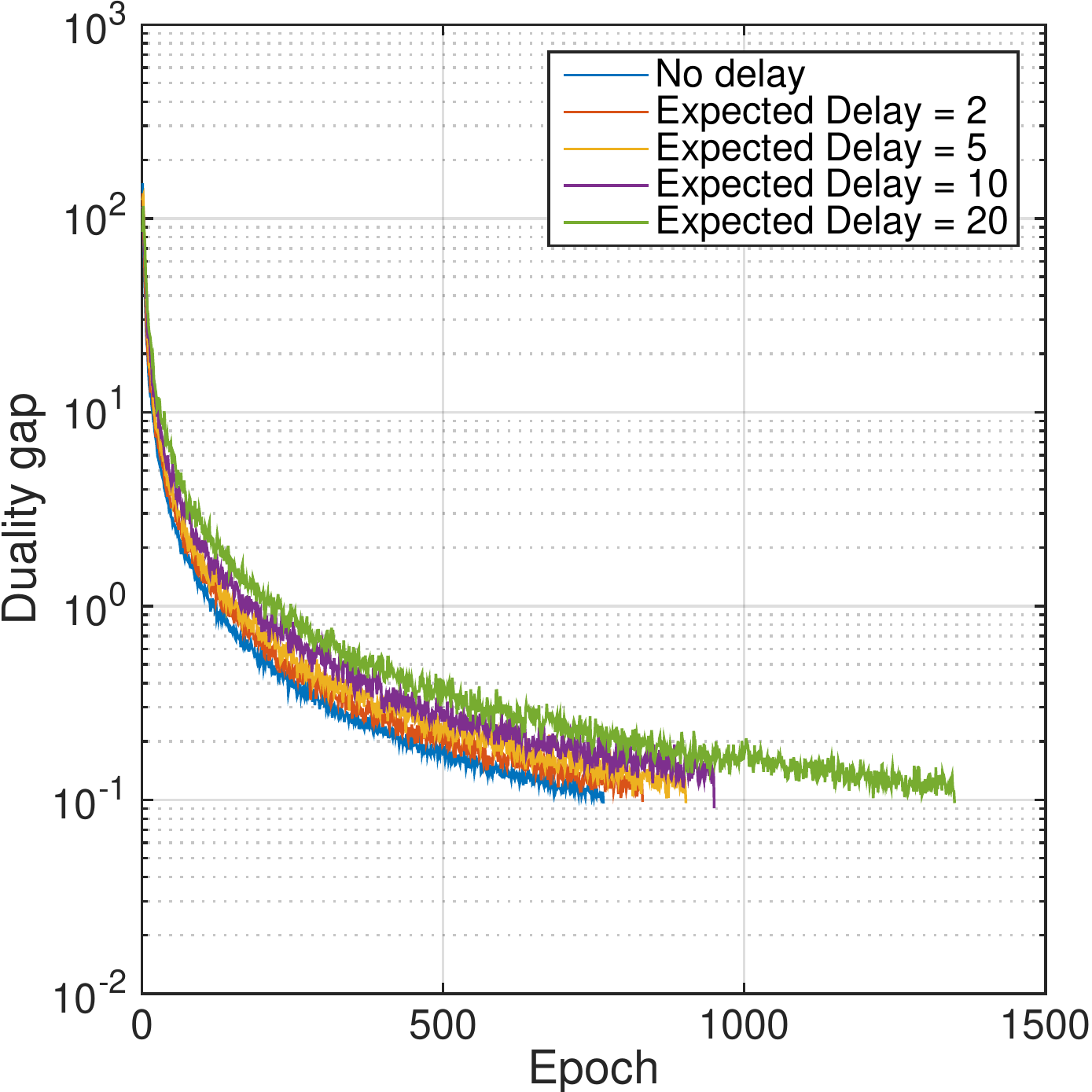}
	\caption{\small Illustrations of the convergence \bcfw with delayed updates. On the left, we have the delay sampled from a Poisson distribution. The figure on the right is for delay sampled from a Pareto distribution. We run each problem until the duality gap reaches $0.1$.} \label{fig:delay}
	\vspace{-0.5cm}
\end{figure}

\vspace{-2mm}
\section{Conclusion}
In this paper, we propose an asynchronous parallel generalization of the block-coordinate Frank-Wolfe method~\citep{lacoste2013block} and provide intuitive conditions under which it has a provable speed-up over \bcfw. The asynchronous updates allow our method to be robust to stragglers and node failure as the speed of \algo depends on {\em average} worker speed instead of the slowest. We demonstrate the effectiveness of the algorithm in structural SVM and Group Fused Lasso with both controlled simulation and real-data experiments on a multi-core workstation. For the structural SVM, it leads to a speed-up over the state-of-the-art \bcfw by an order of magnitude using 16 parallel processors. As a projection-free Frank-Wolfe method, we expect our algorithm to be very competitive in large-scale constrained optimization problems, especially when projections are expensive. Future work includes analysis for the strongly convex case and ultimately releasing a carefully implemented software package for practitioners to deploy in Big Data applications.


\bibliographystyle{apa-good}
\bibliography{FrankWolfe}

\appendix
\section{Convergence analysis}\label{sec:conv_proof}
We provide a self-contained convergence proof in this section. The skeleton of our convergence proof follow closely from \citet{lacoste2013block} and \citet{jaggi2013revisiting}. There are a few subtle modification and improvements that we need to add due to our weaker definition of approximate oracle call that is nearly correct only in expectation. The delayed convergence is new and interesting for the best of our knowledge, which uses a simple result in ``load balancing'' \citep{mitzenmacher2001power}.

 Note that for the cleanness of the presentation, we focus on the primal and primal-dual convergence of the version of the algorithms with pre-defined step sizes and additive approximate subroutine, it is simple to extend the same analysis for line-search variant and multiplicative approximation.

\subsection{Primal Convergence}

\begin{lemma}\label{lem:recursion_PFW}
	Denote the gap between current $f(x^{(k)})$ and the optimal $f(x^*)$ to be $h(x^{(k)})$. The iterative updates in Algorithm~\ref{alg:PFW}
	(with arbitrary fixed stepsize $\lambda$ or by the coordinate-line search) obey
	$$
	\E h(x^{(k+1)}) \leq  (1-\frac{\gamma\tau}{n})\E h(x^{(k)})  + \frac{\gamma^2(1+\delta)}{2}C_f^\tau.
	$$
	where the expectation is taken over the joint randomness all the way to iteration $k+1$.
\end{lemma}
\begin{proof}
	Let $x := x^{(k)}$ for notational convenience. We prove the result for Algorithm ~\ref{alg:PFW} first. Apply the definition of $C_f^{(S)}$ and then apply the definition of the additive approximation in \eqref{eq:approx_add}, to get
	\begin{align*}
	f(x^{(k+1)}_{\mathrm{line-search}}) &\leq f(x^{(k+1)}_{\gamma}) = f(x+\gamma\sum_{i\in S}( s_{[i]} - x_{[i]}))\\
	&\leq f(x) +\gamma\sum_{i\in S}\langle s_{[i]}- x_{[i]} , \nabla_{[i]}f(x)\rangle +\frac{\gamma^2}{2} C_f^{(S)}\\
	&= f(x) +\gamma\langle s_{[S]}- x_{[S]} , \nabla_{[S]}f(x)\rangle +\frac{\gamma^2}{2} C_f^{(S)}
	\end{align*}
	Subtract $f(x^*)$ on both sides we get:
\begin{align*}
h(x^{(k+1)})\leq h(x^{(k)}) +\gamma\langle s_{[S]}- x^{(k)}_{[S]} , \nabla_{[S]}f(x^{(k)})\rangle +\frac{\gamma^2}{2} C_f^{(S)}
\end{align*}
Now take the expectation over the entire history then apply \eqref{eq:approx_add} and definition of the surrogate duality gap \eqref{eq:duality_gap}, we obtain
\begin{align}
\E h(x^{(k+1)})\leq& \E h(x^{(k)}) +\E \left\{\gamma\langle s_{[S]}- x^{(k)}_{[S]} , \nabla_{[S]}f(x^{(k)})\rangle\right\}  +\E\frac{\gamma^2}{2} C_f^{(S)}\nonumber\\
=& \E h(x^{(k)}) +\gamma\E \left\{\langle s_{[S]}, \nabla_{[S]}f(x^{(k)})\rangle - \min_{s\in \cM^{(S)}}\langle s, \nabla_{[S]}f(x^{(k)})\rangle\right\} \nonumber\\
&  - \gamma\E \left\{\langle x^{(k)}_{[S]},\nabla_{[S]}f(x^{(k)})\rangle - \min_{s\in \cM^{(S)}}\langle s, \nabla_{[S]}f(x^{(k)})\right\} +\frac{\gamma^2}{2} C_f^{\tau}\nonumber\\
\leq& \E h(x^{(k)})  + \frac{\gamma^2\delta}{2}C_f^\tau - \gamma\E_{x^{k}} \E_{S|x^{k}} \sum_{i\in S} g^{(i)}(x^{(k)}) +\frac{\gamma^2}{2} C_f^{\tau}\nonumber\\
=& \E h(x^{(k)})  + \frac{\gamma^2\delta}{2}C_f^\tau - \gamma\E_{x^{k}} \frac{\tau}{n} g(x^{(k)}) +\frac{\gamma^2}{2} C_f^{\tau} \label{eq:lem_recursive_g}\\
\leq &  (1-\frac{\gamma\tau}{n})\E h(x^{(k)})  + \frac{\gamma^2(1+\delta)}{2}C_f^\tau.\nonumber
\end{align}
The last inequality follows from the property of the surrogate duality gap $g(x^{(k)})\geq h(x^{(k)})$ due to the fact that $g(x)=f(x)-f^*(\cdot)$. This completes the proof of the descent lemma.
\end{proof}


Now we are ready to state the proof for Theorem~\ref{thm:PFW_primal}.
\begin{proof}[{\bf Proof of Theorem~\ref{thm:PFW_primal}}]
	We follow the proof in Theorem~C.1 in \citet{lacoste2013block} to prove the statement for Algorithm~\ref{alg:PFW}. The difference is that we use a different and carefully chosen sequence of step size.
	
	Take
	$C= h_0 + n(1+\delta)C^{\tau}_f,$
	and denote $\E h(x^{(k)})$ as $h_k$ for short hands.
	The inequality in Lemma~\ref{lem:recursion_PFW} simplifies to
	$$
	h_{k+1} \leq \left(1-\frac{\gamma\tau}{n}\right)h_k + \frac{\gamma^2}{2n} C.
	$$
	Now we will prove $h_k\leq \frac{2nC}{\tau^2 k + 2n}$ for $\gamma_k=\frac{2n\tau}{\tau^2k+2n}$ by induction.
	The base case $k=0$ is trivially true since $C> h_0$. Assuming that the claim holds for $k$, we apply the induction hypothesis and the above inequality is reduced to
	\begin{align*}
	h_{k+1} &\leq (1-\frac{\gamma\tau}{n})h_k + \frac{\gamma^2}{2n} C
	\leq \frac{2nC}{\tau^2 k + 2n} \left[1-\frac{\gamma\tau}{n}+\frac{\tau^2k+2n}{2n}\frac{\gamma^2}{2n}\right]\\
	&=\frac{2nC}{\tau^2 k + 2n}\left[ \frac{\tau^2 k + 2n}{\tau^2 k + 2n} - \frac{2n\tau}{\tau^2 k + 2n}\cdot\frac{\tau}{n} + \frac{(2n\tau)^2}{4n^2(\tau^2k+2n)}\right]\\
	&=\frac{2nC}{\tau^2 k + 2n} \cdot \frac{\tau^2 k + 2n - \tau^2}{\tau^2 k + 2n}
	\leq \frac{2nC}{\tau^2 k + 2n} \cdot \frac{\tau^2 k + 2n - \tau^2 + \tau^2}{\tau^2 k + 2n + \tau^2}\\
	&=\frac{2nC}{\tau^2(k+1)+2n}.
	\end{align*}
	This completes the induction and hence the proof for the primal convergence for Algorithm~\ref{alg:PFW}.
\end{proof}

\subsection{Convergence of the surrogate duality gap}
\begin{proof}[Proof of Theorem~\ref{thm:PFW_primal-dual}]
	We mimic the proof in \citet[Section C.3]{lacoste2013block} for the analogous result closely, and we will use the same notation for $h_k$and $C$ as in the proof for primal convergence, moreover denote $g_k = \E g(x^{(k)})$ 
	First from \eqref{eq:lem_recursive_g} in the proof of Lemma~\ref{lem:recursion_PFW}, we have
	$$
	h_{k+1}\leq h_k - \frac{\gamma \tau}{n} g_k + \frac{\gamma^2}{2n} C.
	$$
	Rearrange the terms, we get
	\begin{equation}\label{eq:pf_primal_dual_eq1}
	g_k\leq \frac{n}{\gamma \tau} (h_k-h_{k+1}) + \frac{\gamma C}{2\tau}.
	\end{equation}
	The idea is that if we take an arbitrary convex combination of $\{g_1,...,g_K\}$, the result will be within the convex hull, namely between the minimum and the maximum, hence proven the existence claim in the theorem. By choosing weight $\rho_k:=k/S_K$ where normalization constant $S_K = \frac{K(K+1)}{2}$ and taking the convex combination of both side of \eqref{eq:pf_primal_dual_eq1}, we have
	\begin{align}
	\E(\min_{k\in [K]}g_k) &\leq \sum_{k=0}^{K}\rho_k g_k \leq \frac{n}{\tau}\sum_{k=1}^{K}\rho_k(\frac{h_k}{\gamma_k}-\frac{h_{k+1}}{\gamma_k}) + \sum_{k=0}^{K}\rho_k\gamma_k\frac{C}{2\tau}  \nonumber\\
	&=\frac{n}{\tau}(\frac{h_0\rho_0}{\gamma_0} - h_{K+1}\frac{\rho_k}{\gamma_k}) + \frac{n}{\tau} \sum_{k=0}^{K-1} h_{k+1} (\frac{\rho_{k+1}}{\gamma_{k+1}}-\frac{\rho_{k}}{\gamma_k}) +   \sum_{k=0}^{K}\rho_k\gamma_k\frac{C}{2\tau}\nonumber\\
	&\leq  \frac{n}{\tau} \sum_{k=0}^{K-1} h_{k+1} (\frac{\rho_{k+1}}{\gamma_{k+1}}-\frac{\rho_{k}}{\gamma_k}) +   \sum_{k=0}^{K}\rho_k\gamma_k\frac{C}{2\tau}\label{eq:pf_primal_dual_eq2}
	\end{align}
	Note that $\rho_0=0$, so we simply dropped a negative term in last line. Applying the step size $\gamma_k=2n\tau/(\tau^2k+2n)$, we get
	\begin{align*}
	\frac{\rho_{k+1}}{\gamma_{k+1}}-\frac{\rho_{k}}{\gamma_k} & = \frac{k+1}{S_K}\frac{\tau^2(k+1)2n}{2n\tau} - \frac{k}{S_K} \frac{\tau^2k+2n}{2n\tau}\\
	&=\frac{1}{2nS_K\tau} \left[\tau^2(k+1)^2+2n(k+1)-\tau^2k^2-2nk\right]\\
	&=\frac{\tau^2(2k+1)+2n}{2nS_K\tau}.
	\end{align*}
	Plug the above back into \eqref{eq:pf_primal_dual_eq2} and use the bound $h_{k+1}\leq 2n C/(\tau^2(k+1) + 2n)$, we get
	\begin{align*}
	\E(\min_{k\in [K]}g_k) &\leq \sum_{k=0}^{K}\rho_k g_k \leq \frac{nC}{\tau^2S_K}
	\sum_{k=0}^{K-1} \frac{\tau^2(2k+2)+2n}{2n}\frac{2n}{\tau^(2k+1)+2n} + \sum_{k=0}^K \frac{k}{S_K}\frac{2n\tau}{\tau^2k+2n}\frac{C}{2\tau}\\
	&=\frac{nC}{\tau^2S_K} \left[\sum_{k=0}^{K-1} (1+\frac{\tau^2}{\tau^2(k+1)+2n})+\sum_{k=1}^{K}\frac{k\tau^2}{(\tau^2k+2n)}\right]\\
	&\leq \frac{nC}{\tau^2S_K} \left[ 2K + K\right] = \frac{2nC}{\tau^2 (K+1)}\cdot 3.
	\end{align*}
	This completes the proof for $K\geq 1$.
\end{proof}

\paragraph{Proof of Convergence with Delayed Gradient}
The idea is that we are going to treat the updates calculated from the delayed gradients as an additive error and then invoke our convergence results that allow the oracle to be approximate. We will first present a lemma that we will use for the proof of Theorem~\ref{thm:delay}.

		\begin{lemma}
		\label{lem:suboptimality}
			Let $x\in \cM $, $\|\cdot\|$ be a norm, $\mathrm{Diam}(\cM)_{\|\cdot\|} \leq D$, $L$ be the gradient Lipschitz constant with respect to the given norm $\|\cdot\|$. Let $\kappa$ be the maximum staleness of the gradient, $\gamma$ be the largest stepsize in the past $\kappa$ steps. Then
			$$\langle \tilde{s}-x, \nabla f(x)\rangle \leq \langle s^*-x, \nabla f(x)\rangle + \gamma \kappa D^2 L$$
			\end{lemma}
			\begin{proof}
				Because $\tilde{s}$ minimizes $\langle s,  \nabla f(\tilde{x}) \rangle$ over $s\in \cH$, we can write
				$$ 
					\langle s^* - \tilde{s},  \nabla f(\tilde{x}) \rangle \geq 0.
				$$
				Using this and H\"{o}lder's inequality, we can write
				$$\langle \tilde{s}-x, \nabla f(x)\rangle- \langle s^*-x, \nabla f(x)\rangle \leq
				 \langle \tilde{s}-s^*, \nabla f(x) -\nabla f(\tilde{x})\rangle \leq \| \tilde{s}-s^*\|\|\nabla f(\tilde{x}) - \nabla f(x)\|_* \leq DL\|\tilde{x} - x\|.$$
				It remains to bound $\|\tilde{x}-x\|$. 
				$$\|\tilde{x} - x\| = \left\|\tilde{x} - \tilde{x} - \sum_{i=1}^{\kappa}\gamma_{-i}(s_{-i}-x_{-i})\right\| \leq \gamma \kappa\max_{i}\|s_{-i}-x_{-i}\| \leq \gamma \kappa D, $$
				where we used the fact that $x$ is at most $\kappa$ steps away from $\tilde{x}$. Assume $\gamma_{-i}$ is the stepsize used and $\langle s_{-i}, x_{-i}\rangle$ are the actual updates that had been performed in the nearest $i$th parameter update before we get to $x$. 
				\end{proof}
				
				
				The second lemma that we need is the following.
				\begin{lemma}\label{lem:convex_combination}
					Let $\cM$ be a convex set. Let $x_0\in \cM$. Let $m$ be any positive integer. For $i=1,...,m$, let $x_{i} = x_{i-1} + \gamma_i(s_i-x_{i-1})$ for some $0\leq \gamma_i\leq 1$ and $s_i\in \cM$ such that, then there exists an $s\in \cM$ and $\gamma\leq \sum_{i=1}^m \gamma_i$, such that $
						x_m  =  \gamma(s-x_0) + x_0.
					$
				\end{lemma}
				\begin{proof}
					We prove by induction. When $m=1$, $s=s_1$ and $\gamma = \gamma_1$.
					Assume for any $m=k-1$, that the claim holds assume the condition is true, then by the recursive formula,
					\begin{align*}
					x_k &= x_{k-1} + \gamma_k(s_{k}-x_{k-1}) \\
						&= x_0 + \gamma (s - x_0)  + \gamma_k [s_{k}  - x_0 - \gamma (s - x_0)]\\
						&= x_0  - (\gamma + \gamma_k - \gamma_k\gamma) x_0 + (\gamma- \gamma_k\gamma) s + \gamma_k s_{k}\\
						&=x_0 + (\gamma + \gamma_k - \gamma_k\gamma) \left[  \frac{\gamma- \gamma_k\gamma}{\gamma + \gamma_k - \gamma_k\gamma}s + \frac{\gamma_k}{\gamma + \gamma_k - \gamma_k\gamma}s_k -x_0\right]\\
						&=x_0 + (\gamma + \gamma_k - \gamma_k\gamma)(s'-x_0)
					\end{align*}
					Note that $s'$ is a convex combination of $s_k$ and $s$ therefore by convexity $s'\in \cM$. Substitute $\gamma\leq \sum_{i=1}^{k-1}\gamma_i$, we get
					$$
					\gamma + \gamma_k - \gamma_k\gamma  \leq  \sum_{i=1}^{k}\gamma_i.
					$$
					This completes the inductive proof for all $m$.
				\end{proof}
				
				The third Lemma that we will need is the following characterization of the expected ``max load'' in randomized load balancing.
				\begin{lemma}[\citep{mitzenmacher2001power,raab1998balls}]\label{lem:max_load}
														Suppose $m$ balls are thrown independently and uniformly at random into $n$ bins. Then, the maximum number of balls in a bin $Y$ satisfies
														$$\E Y \leq \begin{cases}
														 \frac{3\log n}{ \log (n/m)}&\text{ if } m<n/\log n,\\
														c'\log n & \text{ if }m < c n\log n,\\
														\frac{m}{n} + O(\sqrt{\frac{2m}{n}\log n})& \text{ if } m \gg n\log n.
														\end{cases}$$
														where $c'$ is a constant that depends only on $c$.
				\end{lemma}

				\begin{proof}[Proof of Theorem~\ref{thm:delay}]
					The proof involves a sharpening of the Lemma~\ref{lem:suboptimality} for the \bcfw and minibatch setting, where $x\in \cM  = \cM^{(1)}\times ...\times \cM^{(n)}$ is a product domain. The proof idea is to exploit this property. Let the current update be on coordinate block index subset $S$. For each $j\in S$, let the corresponding worker be delayed by $\varkappa_j$, and the corresponding parameter vector be $\tilde{x}$.
					
					As in the proof of Lemma~\ref{lem:suboptimality}, we can bound the suboptimality of the approximate subroutine for solving problem $j:$
					\begin{align}
					\text{Suboptimality}(\tilde{s}_j) &\leq \langle \tilde{s}_j - s^*_j,\nabla_{j}f(\tilde{x})- \nabla_{j}f(x)\rangle 
					\leq \|\tilde{s}_j-s^*_j\|\|\nabla_{j}f(\tilde{x})- \nabla_{j}f(x)\|_* \nonumber\\
					&\leq D_{\|\cdot\|}^{(j)} L_{\|\cdot\|}^{(j)}\|\tilde{x}-x\| =D_{\|\cdot\|}^{(j)} L_{\|\cdot\|}^{(j)} \left\|\sum_{i=1}^{\varkappa_j}\gamma_{-i}(s_{-i}-x_{-i})\right\| \label{eq:delay_proof_der1}\\
					&
					\leq D_{\|\cdot\|}^{1} L_{\|\cdot\|}^{1} \sum_{i=1}^{\varkappa_j}\gamma_{-i}\left\|(s_{-i}-x_{-i})\right\|\nonumber\\
					&\leq D_{\|\cdot\|}^{1} L_{\|\cdot\|}^{1} \sum_{i=1}^{\varkappa_j}\gamma_{-i}D_{\|\cdot\|}^{\tau} \leq \varkappa_j \gamma_{-\varkappa_j} D_{\|\cdot\|}^{1} L_{\|\cdot\|}^{1}D_{\|\cdot\|}^{\tau}.\nonumber
					\end{align}
					Let $\kappa:=\E \varkappa_j$, take expectation on both sides we get 
					$$\E \text{ Suboptimality}(\tilde{s}_j) \leq \E(\varkappa_j \gamma_{-\varkappa_j})  D_{\|\cdot\|}^{1} L_{\|\cdot\|}^{1}D_{\|\cdot\|}^{\tau}$$
					Repeat the same argument for each $i\in S$, we get 
					$$\E \text{ Suboptimality}(\tilde{s}) \leq \E(\varkappa_j \gamma_{-\varkappa_j})  \tau  D_{\|\cdot\|}^{1}  L_{\|\cdot\|}^{1}D_{\|\cdot\|}^{\tau}.$$
					To put it into the desired format in \eqref{eq:approx_add}, we solve the following inequality for $\delta$
					$$\frac{\gamma\delta C_f^\tau}{2} \geq \E (\varkappa_j \gamma_{-\varkappa_j}) \tau   D_{\|\cdot\|}^{\tau} D_{\|\cdot\|}^{1}L_{\|\cdot\|}^1$$
					we get $$\delta \geq \frac{2 \tau}{C_f^\tau}\E \left(\frac{\varkappa_j \gamma_{-\varkappa_j}}{\gamma}\right)    D_{\|\cdot\|}^{\tau} D_{\|\cdot\|}^{1}L_{\|\cdot\|}^1.$$
					By the specification of the stepsizes, we can calculate for each $k$,
					$$\frac{\gamma_{-\varkappa_j}}{\gamma} = \frac{\tau^2 k + 2n}{ \tau^2(\max(k-\varkappa_j,0)) + 2n}.$$
					Note that we always enforce $\varkappa_j$ to be smaller than $\frac{k}{2}$ (otherwise the update is dropped), we can therefore bound $\E (\frac{\varkappa_j \gamma_{-\varkappa_j}}{\gamma})$ by $2\kappa$. This gives us the the first bound \eqref{eq:thm_delay_bound2} on $\delta$ in Theorem~\ref{thm:delay}.
					
					To get the second bound on $\delta$, we start from \eqref{eq:delay_proof_der1} and bound $\|\tilde{x}-x\|$ differently. Let $S$ be the set of $\tau \varkappa_j$ coordinate blocks that were updated in the past $\varkappa_j$ iterations. In the cases where fewer than $\tau\varkappa_j$ blocks were updated, just arbitrarily pick among the coordinate blocks that were updated $0$ times so that $|S|=\tau\varkappa_j$.
					$
					\tilde{x} - x 
					$
					is supported only on $S$. Suppose coordinate block $i\in S$ is updated by $m$ times, as below
					$$
					\tilde{x}_{(i)} = \sum_{j=1}^m \gamma_j (s_j-[x_j]_{(i)})
					$$
					for some sequence of $0\leq \gamma_1,...,\gamma_m\leq 1$ and $s_1,...,s_m \in \cM_i$ and recursively $[x_j]_{(i)} = [x_{j-1}]_{(i)} + \gamma_j (s_j -[x_{j-1}]_{(i)})$ ($x_0 = x$). Apply Lemma~\ref{lem:convex_combination} for each coordinate block, we know that there exist $s_{(i)}\in\cM_i$ in each block $i\in S$ such that 
					$$
					\tilde{x}_{(i)} =  x_{(i)} + \gamma_{(i)} (s_{(i)}  -  x_{(i)})
					$$
					with 
					\begin{equation}\label{eq:bound_step_ratio}
					\gamma_{(i)}\leq \sum_{j\in\text{ iterations where $i$ is updated}}{\gamma_j}\leq  m \gamma_{\max}.
					\end{equation}
					Note that $s_{(i)}\in\cM_i$ for each $i\in S$ implies that their concatenation $s_{(S)} \in \cM_{S}$.
					Also $\gamma_{\max} \leq \gamma_{-\varkappa_j}$. Therefore
$$
\|\tilde{x}-x\| = 
\left\| \sum_{i\in S}  \gamma_i (s_{(i)} - x_{(i)})  \right\|
\leq m \gamma_{\max} \| s_{(S)} - x_{(S)} \|
\leq Y \gamma_{-\varkappa_j} D_{\|\cdot\|}^{\tau\varkappa_j}
$$
					where $Y$ is a random variable that denotes the number of updates received by the most updated coordinate block (the maximum load).			
Apply a previously used argument to get $\gamma_{-\varkappa_j}< 2 \gamma$, take expectation on both sides, to get the following by the law of total expectations and \eqref{eq:bound_step_ratio}
					\begin{align}\label{eq:proof_delay_thm_der}
					\E\|\tilde{x}-x\| &\leq  \E_{\varkappa_j} (\gamma_{-\varkappa_j}D_{\|\cdot\|}^{\varkappa_j\tau} \E_{S_{-\varkappa_j,...,-1} | \varkappa_j}Y) \leq 2\gamma \E_{\varkappa_j} (D_{\|\cdot\|}^{\varkappa_j\tau} \E_{S_{-\varkappa_j,...,-1} | \varkappa_j}Y) .
					\end{align}
					
					By Lemma~\ref{lem:max_load}， when $\kappa_{\max}\tau \leq \frac{n}{\log n}$, it follows from \eqref{eq:proof_delay_thm_der} that
					\begin{align*}
				   \E\|\tilde{x}-x\|\leq 2\gamma\E_{\varkappa_j}{D_{\|\cdot\|}^{\varkappa_j\tau } \frac{3\log n}{\log (n/\varkappa_j\tau)}}
					 &\leq 	\frac{3\log n}{\log [n/(\tau\kappa_{\max})]} 2\gamma \E_{\varkappa_j} D_{\|\cdot\|}^{\varkappa_j\tau }.
					\end{align*}
					When $\kappa_{\max}\tau < cn\log n$, 
					\begin{align*}
					\E\|\tilde{x}-x\|\leq 2\gamma\E_{\varkappa_j}D_{\|\cdot\|}^{\varkappa_j\tau } O(\log n)
					&\leq 	O(\log n) 2\gamma \E_{\varkappa_j} D_{\|\cdot\|}^{\varkappa_j\tau }.
					\end{align*}
					When $\kappa_{\max}\tau \gg n\log n$, then 
					$$\E\|\tilde{x}-x\|\leq (1+o(1))\frac{\tau\kappa_{\max}}{n}  2\gamma \E D_{\|\cdot\|}^{\varkappa_j\tau }.$$

					
					Repeating the above results for each block $j\in S$, and summing them up leads to an upper bound for $\frac{\gamma\delta C_f^{\tau}}{2}$ and the proof of \eqref{eq:thm_delay_bound2} is complete by solving for $\delta$.
				\end{proof}

\section{Proofs of other technical results}

\paragraph{Relationship of the curvatures.}

\begin{proof}[Proof of Lemma~\ref{lem:prop.curvature}]
	$C_f^{(S)}\leq C_f$ follows from the fact that
	$$\langle y_{(S)}-x_{(S)},\nabla_{(S)}f(x)\rangle=\langle y_{[S]}-x_{[S]},\nabla f(x)\rangle,$$
	and $ s_{[S]}\in \cM. $ In other words, the $\arg\sup$ of \eqref{eq:curv_S} is a feasible solution in the $\sup$ to compute the global $C_f$. Similar argument holds for the proof $  C_f^{(i)}\leq C_f^{(S)}$ as $i\in S$. 
	
	In the second part,
	$$C_f^{\tau}=\frac{1}{{n\choose \tau}} \sum_{T\subset [n], |T|=\tau} C_f^{(T)}.$$
	We can evenly partition sets $T$ in the summation into $n$ parts $P_j$ for \( j \in [n] \), such that sets in $P_j$ have the element \(j\). Clearly each \(P_j\) has a size of ${n\choose \tau}/n$.
	We can use \( C_f^{(S)} \geq C_f{(j)} \) from the first inequality of the lemma, to get the inequality below.
	\begin{align*}
	C_f^{\tau}= \frac{1}{{n\choose \tau}} \sum_{j\in[n]} \sum_{T\in P_j} C_f^{(T)} \geq \frac{1}{{n\choose \tau}} \sum_{j\in[n]} \sum_{T\in P_j} C_f^{(j)} = \frac{1}{{n\choose \tau}} \sum_{j\in[n]} {{n}\choose {\tau}} \frac{1}{n} C_f^{(j)} = \frac{1}{n}C_f^{\otimes}\\
	\end{align*}
	
	The relaxation of \( C_f^\tau \) to $C_f$ is trivial since $C_f^{(T)}\leq C_f$ holds for any $T \subseteq [n]$ from the first part of the lemma.
\end{proof}

\paragraph{Bounding $C_f^{\tau}$ using expected boundedness and  expected incoherence}
\begin{proof}[Proof of Theorem~\ref{thm:Cf_master_lemma}]
	By Definition of $H,$ for any $x,z\in \cM$, $\gamma\in [0,1]$
	$$
	f(x+\gamma(z-x)) \leq f(x)+\gamma(z-x)^T\nabla f(x) + \frac{\gamma^2}{2} (z-x)^TH(z-x).
	$$
	Rearranging the terms we get
	$$
	\frac{2}{\gamma^2}\left[f(x+\gamma(z-x))-f(x)-\gamma (z-x)^T\nabla f(x)\right] \leq (z-x)^TH(z-x)
	$$
	The definition of set curvature \eqref{eq:curv_S} is written in an equivalent notation with $z = x_{[S^c]} + s_{[S]}$ and $y=x+\gamma(z-x)=x+\gamma(s_{[S]}-x_{[S]})$. So we know $z-x$ is constrained to be within the coordinate blocks $S$.
	
	
	Plugging this into the definition of \eqref{eq:curv_S} we get an analog of Equation (2.12) in \citet{jaggi2011sparse} for $C_f^{(S)}$.
	\begin{align*}
	C_f^{(S)}&=  \sup_{\begin{subarray}{l} x\in \cM, \gamma\in [0,1]\\ z=x+s_{[S]}\in \cM  \end{subarray},\gamma}\frac{2}{\gamma^2}\left[f(x+\gamma(z-x))-f(x)-\gamma (z-x)^T\nabla f(x)\right]\\
	&\leq \sup_{\begin{subarray}{l} x,z\in \cM, \\ z=x+s_{[S]}  \end{subarray}}
	(z-x)^TH(z-x)
	=\sup_{\begin{subarray}{l} x,z\in \cM,\\ z=x+s_{[S]}  \end{subarray}}
	s_{(S)}^T H s_{(S)}\\
	&\leq \sup_{w\in \cM^{(S)}} (2w^T) H_{S} (2w) = 4\left\{\sup_{w_i\in\cM^{(i)} \forall i\in S} \sum_{i\in S}w_i^T H_{ii} w_i +  \sum_{i,j\in S, i\neq j}w_i^T H_{ii} w_j\right\}\\
	&\leq 4 \left\{ \sum_{i\in S} \sup_z\sup_{w_i}w_i^T H_{ii}(z) w_i +  \sum_{i,j\in S, i\neq j} \sup_z\sup_{w_i,w_j} w_i^T H_{ii}(z) w_j \right\}\\
	&\leq 4(\sum_{i\in S} B_i + \sum_{i,j\in S, i\neq j} \mu_{ij}).
	\end{align*}
	Take expectation for all possible $S$ of size $\tau$ and we obtain the lemma statement.
\end{proof}

\paragraph{Proof of specific examples}
\begin{proof}[Proof of Example~\ref{eg:structSVM_avg}]
	First of all, $H=\lambda A^TA$. Since all columns of $A$ have the same magnitude $\sqrt{2}/n$. By the Holder's inequality and the $1$-norm constraint in every block, we know $B_i=\frac{2}{n^2\lambda}$ for any $i$ therefore $B=\frac{2}{n^2\lambda}$.
	Secondly, by well-known upper bound for the area of the spherical cap, which says for any fixed vector $z$ and random vector $a$ on a unit sphere in $\R^d$,
	$$
	\P(|\langle z,a\rangle| > \epsilon\|z\|) \leq 2e^{\frac{-d\epsilon^2}{2}},
	$$
	we get
	$$
	\P(\mu_{ij} > 2\sqrt{\frac{20\log d}{d}}) \leq \frac{2}{d^{10}}.
	$$
	Take union bound over all pairs of labels we get the probability as claimed.
\end{proof}

\begin{proof}[Proof of Example~\ref{eg:group_fused_lasso}]
	The matrix $D^TD$ is tridiagonal with $2$ on the diagonal and $-1$ on the off-diagonal.
	If we vectorize $U$ by concatenating $u=[u_1;...;u_{n-1}]$, the Hessian matrix for $u$ will be $H=\Pi I_{d} \otimes (D^TD) \Pi^T$ where $\Pi$ is some permutation matrix. Without calculating it explicitly, we can express
	\begin{align*}
	&u_S^TH_Su_S = u_S^T (D^T\otimes 1_d) (D^T\otimes 1_d)^T u_S\\
	=& \sum_{i\in S} u_i^T \begin{bmatrix}
	D_{:,i}^T \\
	D_{:,i}^T \\
	\vdots \\
	D_{:,i}^T \\
	\end{bmatrix}
	\begin{bmatrix}
	D_{:,i} &
	D_{:,i}&
	\hdots &
	D_{:,i}
	\end{bmatrix} u_i
	+ \sum_{i,j\in S, i\neq j} u_i^T \begin{bmatrix}
	D_{:,i}^T \\
	D_{:,i}^T \\
	\vdots \\
	D_{:,i}^T \\
	\end{bmatrix}
	\begin{bmatrix}
	D_{:,j} &
	D_{:,j}&
	\hdots &
	D_{:,j}
	\end{bmatrix} u_j.
	\end{align*}
	We note that for any $|i-j|\geq 2$, the second term is $0$. Apply the constraint that $\|u_i\|_2\leq \lambda$ and the fact that the $\ell_2$ operator norm of $\begin{bmatrix}D_{:,j} &
	D_{:,j}&
	\hdots &
	D_{:,j}
	\end{bmatrix}$ is $\sqrt{2d}$, we get $B_i= 2\lambda^2 d$. Similarly, $2(n-2)$ nonzero obeys $\mu_{ij}= \lambda^2 d$.
	This allows us to obtain an upper bound
	$$
	C_f^{\tau}\leq 2\tau\lambda^2d + \frac{2(n-2)\tau(\tau-1)}{(n-2)(n-1)} \lambda^2d \leq 4\tau \lambda^2d.
	$$
	which scales with $\tau$.
\end{proof}

\subsection{Pseudocode for the Multicore Shared Memory Architecture}
We present the pseudocode for the multicore shared memory setting in Algorithm~\ref{alg:SPFW}. It is the same as Algorithm~\ref{alg:PFW} except that each worker becomes a thread, the network buffer of servers become the a data structure, the workers' network buffer becomes the shared parameter vector and the workers can write to the data structure or the shared parameter vector directly. 
\begin{algorithm}[tbh]
	\caption{\algo: Asynchronous Parallel Block-Coordinate Frank-Wolfe (Shared memory)}
	\label{alg:SPFW}
	\begin{algorithmic}
		\STATE{\textsc{------------------------Server thread---------------------}}
		\STATE {\bfseries Input:} An initial feasible $x^{(0)}$, mini-batch size $\tau$, number of workers $T$.
		\STATE{0. } Write $x^{(0)}$ to shared memory. Declare a container (a queue or a stack).
		\FOR{$k$ = 1,2,... ($k$ is the iteration number.)}
		\STATE{1. } Keep popping the container until we have $\tau$ updates on $\tau$ disjoint blocks. Denote the index set by $S$.
		\STATE{2. } Set step size $\gamma=\frac{2n\tau}{\tau^2k+ 2n}$.
		\STATE{3. } Write sparse updates $x^{(k)} = x^{(k-1)} + \gamma \sum_{i\in S}( s_{[i]} - x_{[i]}^{(k-1)})$ into the shared memory.
		\IF{converged}
		\item Broadcast STOP signal to all threads and break.
		\ENDIF
		\ENDFOR
		\STATE {\bfseries Output: } $x^{(k)}$.
		
		\STATE{\textsc{-----------------------Worker threads---------------------}}
		\WHILE {no STOP signal received}
		\STATE{a.} Randomly choose $i\in [n]$. 
		\STATE{b.} Calculate partial gradient $\nabla_{(i)}f(x)$ using $x$ in the shared memory and solve \eqref{eq:exact_oracle}.
		\STATE{c.} Push $\{i,s_{(i)}\}$ to the container.
		\ENDWHILE
	\end{algorithmic}
\end{algorithm}
The above pseudo code can be further simplified when $\tau = 1$ (see Algorithm~\ref{alg:LPFW}). In particular, we do not need a server any more. Each worker can simply write to the shared memory bus. The probability of two workers writing to the same block is small as we analyzed in Section~\ref{sec:distr}. If updates to each coordinate block is atomic, then this is essentially lock-free as in \citet{niu2011hogwild} (\citet{niu2011hogwild} is stronger in that it allows each scalar addition to be atomic).	
\begin{algorithm}[tbh]
	\caption{\algo: Asynchronous Parallel Block-Coordinate Frank-Wolfe (Lock-Free Shared-Memory)}
	\label{alg:LPFW}
	\begin{algorithmic}
		\STATE {\bfseries Input:} An initial feasible $x^{(0)}$, number of workers $T$, a centralized counter.
		\STATE{0. } Write $x^{(0)}$ to shared memory.
		\STATE{\textsc{------------independently on each thread-----------}}
		\WHILE {not converged}
		\STATE{a.} Randomly choose $i\in [n]$. 
		\STATE{b.} Calculate partial gradient $\nabla_{(i)}f(x)$ using $x$ in the shared memory and solve \eqref{eq:exact_oracle}.
		\STATE{c.} Read centralized counter for $k$. Set step size $\gamma=\frac{2n}{k+ 2n}$.
		\STATE{d.} Add $\gamma (s_{(i)} - x_{(i)} )$ to block $i$ of the shared memory.
		\STATE{e.} Increment the counter $k=k+1$.
		\ENDWHILE
		\STATE{\textsc{-------------------------------------------------------------------}}
		\IF{converged}
		\STATE {\bfseries Output: } $x^{(k)}$. and break.
		\ENDIF
	\end{algorithmic}
\end{algorithm}

\section{Application to Structural SVM}
\label{structSVM}
We briefly review structural SVMs and show how to solve the associated convex optimization problem using our \algo method. 

In structured prediction setting, the task is to predict a structured output $\y \in \cY$, given $\x \in \cX$. For example, $\x$ could be the pixels in the picture of a word, $\y$ could be the sequence of characters in the word.
A feature map $\phi : \cX \times \cY \rightarrow \R^d$ encodes compatibility between inputs and outputs. A linear classifier parameter $\w$ is learned from data so that $\argmax_{\y \in \cY}\langle \w, \phi(\x,\y) \rangle$ gives the output for an input $\x$.
Suppose we have the training data $\{\x_i,\y_i\}_{i=1}^n$ to learn $\w$.
Define $\psi_i(\y) := \phi(\x_i,\y_i)-\phi(\x_i,y)$ and let $L_i(\y):= L(\y_i,\y)$ denote the loss incurred by predicting $\y$ instead of the correct output $\y_i$.
The classifier parameter $\w$ is learned by solving the optimization problem
\begin{align}\label{eq:structsvm_primal}
&\min_{\w, \xi} \frac{\lambda}{2} \| \w \|^2 + \frac{1}{n} \sum_{i=1}^n \xi_i \\
&\text{s.t}\hspace{0.2cm} \langle \w, \psi_i(\y) \rangle \geq L(\y_i, \y) -\xi_i \hspace{0.5cm} \forall i,\y \in \cY(\x_i). \nonumber
\end{align}

We solve the dual of this problem using our method. We introduce some more notation to formulate the dual. Denote $\cY_i := \cY(\x_i)$, the set of possible labels for $\x_i$. Note that $|\cY_i|$ is exponential in the length of label $\y_i$. Let $m=\sum_{i=1}|\cY_i|$.
Let $A\in \R^{d\times m}$ denote a matrix whose $m$ columns are given by $ \{\frac{1}{\lambda n} \psi_i(\y)\mid i\in [n], \y \in \cY_i\}$.
Let $b\in \R^m$ be a vector given by the entries $\{ \frac{1}{n} L_i(\y) \mid i\in [n], \y \in \cY_i \}$.
The dual of (\ref{eq:structsvm_primal}) is given by
\begin{align}\label{eq:structsvm_dual}
&\min_{\alpha \in \R^m} f(\alpha) := \frac{\lambda}{2} \| A \alpha \|^2 - b^T\alpha \\
&\text{s.t}\hspace{0.2cm} \sum_{\y \in \cY_i} \alpha_i(\y) = 1 \hspace{0.5cm}\forall i \in [n], \alpha \geq 0 \nonumber
\end{align}
The primal solution $\w$ can be retrieved from the dual solution $\alpha$ from the relation $\w = A\alpha$ obtained from KKT conditions. Also note that the domain $\cM$ of (\ref{eq:structsvm_dual}) is exactly the product of simplices $\cM = \Delta_{|\cY_1|} \times \cdots \times \Delta_{|\cY_n|}$.

The subproblem in equation (\ref{eq:exact_oracle}) takes a well-known form in the Frank-Wolfe setup for solving (\ref{eq:structsvm_dual}). The gradient is given by
\begin{align*}
\nabla f(\alpha) = \lambda A^TA\alpha -b= \lambda A^T \w - b
\end{align*}
whose $(i,\y)$-th component is given by $\frac{1}{n}\left( \langle \w, \psi_i(\y) \rangle - L_i(\y)\right)$. Define $H_i(\y;\w):= L_i(\y) - \langle \w, \psi_i(\y)\rangle$ so that the $(i,\y)$-th component of the gradient is $-\frac{1}{n} H_i(\y;\w)$.
In the subproblem (\ref{eq:exact_oracle}), the domain $\cM^{(i)}$ is the simplex $\Delta_{\cY_i}$ and the block gradient $\nabla_{(i)} f(\alpha)$ is linear. So, the objective is minimized at a corner of the simplex $\cM^{(i)}$ and the optimum value is simply given by $\min_{\y} \nabla_{(i)} f(\alpha) $ which can be rewritten as $\max_{\y} H_i(\y;\w)$. Further, the corner can be explicitly written as the indicator vector $e^{\y_i^*} \in \cM^{(i)}$ where $\y_i^* = \argmax_{\y} H_i(\y;\w)$. It turns out that this maximization problem can be solved efficiently for several problems. For example, when the output is a sequence of labels, a dynamic programming algorithm like Viterbi can be used.

As mentioned before, $m$ is too large to update the dual variable $\alpha$ directly. So, we make an update to the primal variable $\w = A\alpha$ instead. The Block-Coordinate Frank-Wolfe update for the $i$-th block maybe written as $\alpha_{(i)}^{k+1} = \alpha_{i}^k + \gamma ( s_{i} - \alpha_{(i)}^k) $ where $\gamma$ is the step-size. Recalling that the optimal $s_{i}$ is $e^{\y_i^*}$, by multiplying the previous equation by $A_i$, we arrive at $\w_i^{(k+1)} = \w_{i}^k + \gamma( A_{i,y_i^*} - \w_{i}^{(k)})$ where $\w_i^{(k)}:= A_i \alpha_{(i)}$. From this definition of $\w_i^{(k)}$, the primal update is obtained by noting that $\w^{(k)} = \sum_i \w_i^{(k)}$. Explicitly, the primal update is given by $\w^{(k+1)} = \w^k + \gamma( A_{i,y_i^*} - \w_{i}^{(k)})$. Note that $A_{i,y_i^*} = \frac{1}{\lambda n} \psi_(\y_i^*)$. This Block-Coordinate version can be easily extended to \algo. In our shared memory implementation, for OCR dataset, we do the line search computation and $\w_i^{(k)}$ update step on the workers instead of the server because these computations turn out to be expensive enough to make the server the bottleneck even for modest number of workers.

\section{Other technical results and discussions}

\subsection{Controlling collisions in distributed setting}
\label{sec:distr}

In the distributed setting, different workers might end up working on the same slot. 

In Algorithm~\ref{alg:PFW},  different workers may end up working on the same coordinate block and the server will drop a number of updates in case of collision. The following proposition shows that for this potential redundancy is not excessive is small and for a large range of $\tau$, we also show additional strong concentration to its mean.
\begin{proposition}\label{prop:work_per_epoch}
	In the distributed asynchronous update scheme above:
	\begin{enumerate}[label=\roman*), leftmargin=*, topsep=-5pt,itemsep=-12pt]
		\item The expected number of subroutine calls from all workers to complete each iteration is $\tau+\sum_{i=1}^{\tau-1}\frac{i}{n-i}$.\\ 
		\item If $0.02n<\tau< 0.6n$, with probability at least $1-\exp(-n/60)$, no more than $2\tau$ random draws ($2\tau$ subroutine calls in total from all workers) suffice to complete each iteration.
	\end{enumerate}
\end{proposition}
\begin{proof}
	The first claim is the well-known coupon collector problem.
	
	The second claim requires an upper bound of the expectation. In expectation, we need $\frac{n}{n-k}$ balls to increase the unique count from $k$ to $k+1$. So in expectation we need
	\begin{align*}
	1+\frac{n}{n-1} + \frac{n}{n-2}+...+\frac{n}{n-\tau+1}= \tau+\sum_{i=1}^{\tau-1}\frac{i}{n-i}\\
	\leq \tau+\frac{1+2+\cdots+(\tau-1)}{n-\tau+1}=\tau+\frac{\tau(\tau-1)}{2(n-\tau+1)}<\tau\left[1+\frac{1}{2(n/\tau-1)}\right].
	\end{align*}

	To see the second claim, first defined $f_{t}$ to be the number of non-empty bins after $t$ random ball throws, which can be consider as a function of the $t$ iid ball throws $X_1, X_2,..., X_t$.
	It is clear that if we change only one of the $X_i$, $f_{t}$ can be changed by at most $1$.
	Also, note that the probability that any one bin being filled is $1-(1-\frac{1}{n})^t$, so
	$\E f_t = n\left[1-\left(1-\frac{1}{n}\right)^t\right].$
	
	By the McDiarmid's inequality, $
	\P\left[ f_t < \E f_t - \epsilon\right] \leq \exp{\left[-\frac{2\epsilon^2}{t}\right]}.$
	Take $t=2\tau$, and $\epsilon= \E f_{2\tau}-\tau$, then
	\begin{align*}
	\P\left[ f_{2\tau} < \tau\right]\leq& \exp{\left[-\frac{1}{\tau}\left(n\left[1-\left(1-\frac{1}{n}\right)^{2\tau}\right] - \tau\right)^2\right]} \leq \exp{\left[-\frac{1}{\tau}\left(n\left[1-e^{-\frac{2\tau}{n}}\right] - \tau\right)^2\right]} \\
	=& \exp{\left[-n\cdot\frac{n}{\tau}\left(1-e^{-\frac{2\tau}{n}} - \frac{\tau}{n}\right)^2\right]} \leq \exp{[-Cn]},
	\end{align*}
	where $C$ is some constant which is the smaller of the two evaluations of the function $\frac{n}{\tau}\left(1-e^{-\frac{2\tau}{n}} - \frac{\tau}{n}\right)^2$ at $\tau=0.02n$ and $\tau=0.6n$ (where the function is concave between the two). As a matter of fact, $C$ can be taken as $\frac{1}{60}$.
	
	Let $g_{\tau}$ be the number of balls that one throws that fills $\tau$ bins, the result is proven by noting that
	$$\P(g_{\tau}\leq 2\tau) = \P(f_{2\tau}\geq \tau) \geq 1- \exp{[-Cn]}.$$
\end{proof}



\subsection{Curvature and Lipschitz Constant}\label{sec:lipschitz_curvature}
In this section, we illustrate the relationship between the coordinate curvature constant, coordinate gradient Lipschitz conditions, and work out the typical size of the constants in Theorem~\ref{thm:delay}. For the sake of discussion, we will focus on the quadratic function $f(x)  = \frac{x^TAx}{2} + bTx$. We start by showing that for quadratic function. The constant that one can get via choosing a specific norm can actually match the curvature constant.
\begin{proposition}
	For quadratic functions with Hessian $A\succeq 0$, there exists a norm $\|\cdot\|$ such that the curvature constant $C_f  = [D_{\|\cdot\|}]^2L_{\|\cdot\|}$.
\end{proposition}
\begin{proof}
	We will show that this norm is simply the $A$-norm, $\|\cdot\|_A  = \sqrt{(\cdot)^T A (\cdot) }$.
	The upper bound $C_f \leq [D_{\|\cdot\|_A}]^2L_{\|\cdot\|_A}$ is a direct application of the result in~\citep[Appendix D]{jaggi2013revisiting}.
	To show a lower bound it suffices to construct $s,x\in \cM, \gamma \in [0,1]$ and $y= \gamma s+ (1-\gamma)x$ such that 
	$$\frac{2}{\gamma^2} (f(y)  - f(x)  - \langle y-x, \nabla f(x)\rangle) =  [D_{\|\cdot\|_A}]^2L_{\|\cdot\|_A}.$$

	For quadratic functions, 
	$$\frac{2}{\gamma^2}[f(y)  - f(x)  - \langle y-x, \nabla f(x)\rangle]  = \frac{1}{2}(y-x)^TA(y-x) = \frac{1}{\gamma^2}\|y-x\|_A^2$$
	Take $\gamma = 1$ and $y,x$ on the boundary of $\cM$ such that $\|y-x\|_A = D_{\|\cdot\|_A}$, as a result, we get 
	$
	C_f \geq D_{\|\cdot\|_A}]^2.
	$
	It remains to show that the gradient Lipschitz constant with respect to $\cA$-norm is $1$, which directly follows from the Taylor expansion.	
\end{proof}
Similar arguments work for $C_f^{(i)}$ and $C_f^{(S)}$ under the same norm. Clearly, this means that the corresponding restriction of the subset domain has $A_{i,i}$-norm or $A_{(S)}$-norm.

We now consider the approximation constants due to the delays in Theorem~\ref{thm:delay}, and work out more explicit bounds for quadratic functions and carefully chosen norm.  Recall that the simple bound \eqref{eq:thm_delay_bound1} has constant $\delta$ in the order of 
$$\frac{\kappa \tau L_{\|\cdot\|}^1 D_{\|\cdot\|}^{1} D_{\|\cdot\|}^{\tau}}{C_f^\tau}.$$
Suppose we use the  $A$-norm, then $L_{\|\cdot\|}^1 = L_{\|\cdot\|}^\tau= 1$, and $C_f^\tau = [D_{\|\cdot\|}^{\tau}]^2$, the bound can be reduced to
$$
\delta = O(\frac{\tau D_{\|\cdot\|}^{1}}{D_{\|\cdot\|}^{\tau}}) = O(\kappa\sqrt{\tau}).
$$
where the last step requires $\cM_i$ to be all equivalent and $A$ to be block-diagonal with identical $A_{(i)}$.

Similarly the strong bound \eqref{eq:thm_delay_bound2} has constant $\delta$ in the order of 
$$
\delta = \tilde{O}\left(\frac{\tau L_{\|\cdot\|}^1 D_{\|\cdot\|}^{1} D_{\|\cdot\|}^{\kappa\tau}}{C_f^\tau}\right) = \tilde{O}\left(\frac{\tau D_{\|\cdot\|}^{1} D_{\|\cdot\|}^{\kappa\tau}}{[D_{\|\cdot\|}^{\tau}]^2}\right) = \tilde{O}(\sqrt{\kappa\tau})
$$
Again, the last step requires a strong assumption that $\cM_i$ to be all equivalent and $A$ to be block-diagonal with identical diagonal blocks. While these calculations only apply to specific case of a quadratic function with a lot of symmetry, we conjecture that in general the flexibility of choosing the norm will allow the ratio of these boundedness constants and $C_f^\tau$ to be a well-controlled constant and the typical dependency on the system parameter $\tau$ and $\kappa$ should stay within the same ball park.

\subsection{Examples and illustrations}

\begin{example}[Structural SVM worst-case bound]\label{eg:structSVM_worst}
	For structural SVM with arbitrary data (including even pathological/trivial data), using notation from Lemmas~A.1~and~A.2 of \citet{lacoste2013block}, define $R:=\max_{i\in[n], y\in\cY_i} \|\psi_i(y)\|_2$. Then we can provide an upper bound
	\begin{equation}\label{eq:C_f^tau_bound}
	B, \mu \leq \frac{R^2}{\lambda n^2}\quad\implies\quad
	C_f^{\tau} \leq \frac{4\tau^2R^2}{\lambda n^2}.
	\end{equation}
	In this case, for any $\tau=1,...,n$, the rate of convergence will be the same $O(\frac{R^2}{\lambda k})$.
\end{example}

\paragraph{An illustration for the group fused lasso}
Figure~\ref{fig:signalIllustration} shows a typically application for group fused lasso (filtering piecewise constant multivariate signals whose change poitns are grouped together).
\begin{figure}[h]
	\centering
	\includegraphics[width=1\linewidth]{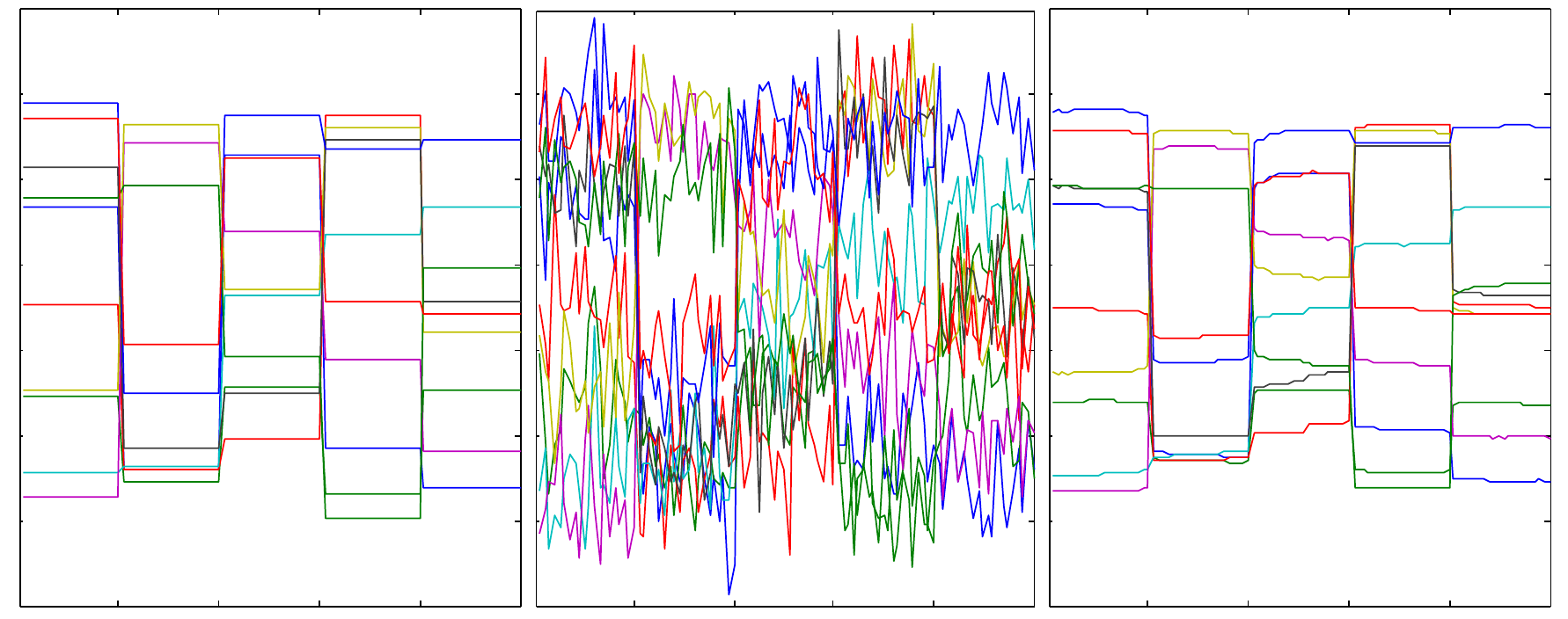}
	\caption{Illustration of the signal data used in the Fused Lasso experiments. We show the
		original signal (left), the noisy signal given to the algorithm (middle), and the
		signal recovered after performing the fused lasso optimization (right). }
	\label{fig:signalIllustration}
\end{figure}


\subsection{Comparison to parallel block coordinate descent}\label{sec:cmp_richtarik}
With some understanding on $C_f^{\tau}$, we can now explicitly compare the rate of convergence in Theorem~\ref{thm:PFW_primal} with parallel BCD  \citep{richtarik2012parallel,liu2013asynchronousCD} under the assumption of $\mu = O(B/\tau)$ --- a fair and equally favorable case to all of these methods. To facilitate comparison, we will convert the constants in all three methods to block coordinate gradient Lipschitz constant $L_i$, which obeys
\begin{equation}\label{eq:Li}
f(x+s_{[i]}) \leq f(x) + \langle s_{[i]}, \nabla f(x)\rangle +  L_i \|s_{[i]}\|^2,
\end{equation}
for any $x\in \cM, s_{(i)}\in \cM_i$. Observe that
$B_i$ $\leq$  $4L_i \mathrm{diam}(\cM_i)^2$ $=$ $L_i \max_{x_i^*,x_i\in\cM_i}\|x_i - x^*_i\|^2$, so
\begin{align}
B &\leq \frac{1}{n}\sum_{i} L_i \max_{x_i,x^*_i}\|x_i-x^*_i\| \\
&\leq \frac{1}{n}\sum_{i}L_i \max_x \|x-x^*\|^2=\E_i(L_i) R^2
\end{align}
where $ R:=\max_x \|x-x^*\|$. The rate of convergence for the three methods (with $\tau$ oracle calls considered as one iteration) are given below.
\begin{table}[H]
	\centering
	\begin{tabular}{|c|c|}
		\hline
		Method & Rate \\ \hline
		\algo (Ours) & 
		$O_p\left(\frac{n\E_i(L_i)R^2}{\tau k}\right)$ \\ \hline
		P-BCD\tablefootnote{In \citet[Theorem~19]{richtarik2012parallel}} & $O_p\left(\frac{n \E_i(L_i)R^2}{\tau k}\right)$ \\ \hline
		AP-BCD\tablefootnote{In \citet[Theorem~3]{liu2013asynchronousCD}} & $O_p\left(\frac{n \max_iL_iR^2}{\tau k}\right)$ \\ \hline
	\end{tabular}
\end{table}
The comparison illustrates that these methods have the same $O(1/k)$ rate and almost the same dependence on $n$ and $\tau$ despite the fact that we use a much simpler linear oracle. Nothing comes for free though: Nesterov acceleration does not apply for Frank-Wolfe based methods in general, while a careful implementation of parallel coordinate descents can achieve $O(1/k^2)$ rate without any full-vector interpolation in every iteration \citep{fercoq2015accelerated}. Also, Frank-Wolfe methods usually need additional restrictive conditions or algorithmic steps to get linear convergence for strongly convex problems.

These facts somewhat limits the applicability of our method to cases when projection can be computed as efficiently as \eqref{eq:exact_oracle}. However, as is surveyed in \citep{jaggi2013revisiting}, there are many interesting cases when \eqref{eq:exact_oracle} is much cheaper than projections, e.g., projection onto a nuclear norm ball takes $O(n^3)$ while \eqref{eq:exact_oracle} takes only $O(n^2)$.

\end{document}